\documentclass[sn-mathphys-num]{sn-jnl}

\usepackage{graphicx}%
\usepackage{multirow}%
\usepackage{amsmath,amssymb,amsfonts}%
\usepackage{amsthm}%
\usepackage{mathrsfs}%
\usepackage[title]{appendix}%
\usepackage{xcolor}%
\usepackage{textcomp}%
\usepackage{manyfoot}%
\usepackage{booktabs}%
\usepackage{algorithm}%
\usepackage{algorithmicx}%
\usepackage{algpseudocode}%
\usepackage{listings}%
\usepackage{url}%
\usepackage{xcolor}%
\usepackage{accents}%
\usepackage{float}%
\usepackage{verbatim}%
\usepackage{cancel}%
\usepackage[font=small,skip=0pt]{caption}
\numberwithin{equation}{section}
\newtheorem{t1}{Proposition}[section]
\newtheorem{t2}{Corollary}[section]
\newtheorem{l1}{Lemma}[section]
\newcommand{\ignore}[1]{}

\theoremstyle{thmstyleone}%
\newtheorem{theorem}{Theorem}
\theoremstyle{thmstyletwo}%

\theoremstyle{thmstylethree}%

\begin{document}

\title[Article Title]{Point Prediction for Streaming Data}


\author[1]{\fnm{Aleena} \sur{Chanda}}\email{achanda2@huskers.unl.edu}

\author[2]{\fnm{N. V.} \sur{Vinodchandran}}
\email{vinod@unl.edu}

\author[,1]{\fnm{Bertrand} \sur{Clarke}}\email{bclarke3@unl.edu}
\equalcont{Corresponding author; this work was mainly done by the first author under the supervision of the second and third.}

\affil[1]{\orgdiv{Department of Statistics}, \orgname{U. Nebraska-Lincoln}, \orgaddress{\street{340 Hardin Hall North Wing}, \city{Lincoln}, \postcode{68583-0963}, \state{NE}, \country{USA}}}

\affil[2]{\orgdiv{School of Computing}, \orgname{U. Nebraska-Lincoln}, \orgaddress{\street{Avery Hall, 1144 T St Suite 256}, \city{Lincoln}, \postcode{68508}, \state{NE}, \country{USA}}}



\abstract{We present two new approaches for point prediction with streaming data.
One is based on the Count-Min sketch (CMS)
and the other is based on Gaussian process priors
with a random bias.  These methods are
intended for the most general predictive problems
where no true model can be usefully 
formulated for the data stream. In statistical 
contexts, this is often called the 
$\mathcal{M}$-open problem class. 
Under the assumption that the data consists of i.i.d samples from a fixed distribution function $F$, we show that the CMS-based
estimates of the distribution function 
are consistent.

We compare our new methods with two
established predictors in terms of
cumulative
$L^1$ error.  One is
based on the Shtarkov solution 
(often called the
normalized maximum likelihood) in the normal
experts setting and the
other is based on Dirichlet process priors.
These comparisons are for two cases.  The first
is one-pass meaning that the updating
of the predictors is done using the fact
that the CMS is a sketch.  For predictors
that are not one-pass, we use streaming
$K$-means to give a representative
subset of fixed size that can be updated as
data accumulate.  

Preliminary computational work suggests that the one-pass median version of the CMS method is 
rarely outperformed by the other methods 
for sufficiently complex data.  We also
find that predictors based on Gaussian process priors with random biases perform well.
The Shtarkov predictors we use here did not
perform as well probably because we were only
using the simplest example.  The other predictors
seemed to perform well mainly when the data
did not look like they came from an
${\cal{M}}$-open data generator.
}




\maketitle

\section{Problem Formulation}
\label{Intro}

Consider a string of bounded real numbers, say
$y^n = (y_1, y_2, \cdots, y_n, \cdots)$ with
$y_i \in [m,M]$ $\forall i$ for some $M>m>0$ and
suppose our goal is sequential prediction.  
That is, we want to form a good predictor 
$\hat{Y} = \hat{Y}_{n+1}(y^n)$
of $y_{n+1}$.  This is often called prediction along
a string or streaming data 
when no assumptions can be made about 
the distributional properties of the $y_i$'s -- apart, here, from boundedness.

One of the earliest contributions to this class of
problems, chiefly in the classification context,
was \cite{Barron:Haussler:1993}
and it was the main topic of the celebrated book
by \cite{Cesa-Bianchi:Lugosi:2006}.
Prediction along a string roughly corresponds to 
${\cal{M}}$-open problems in the sense of 
\cite{Bernardo:Smith:2000}, see also \cite{Le:Clarke:2017}.
When we say that there simply is no true model for the 
data, we are effectively
forced into the {\em prequential} setting of \cite{Dawid:1984};
for a more recent exposition see \cite{Vovk:Shen:2010}.
There is little systematic work on prequential prediction
for $\mathcal{M}$-open data even though one can 
argue this is
the most important setting.

Here we propose two new predictors for $\hat{Y}$.
The first is based on using $y^n$ to form an estimated 
empirical distribution function (EDF).
That is, we use the data up to a time $n$ 
to define an EDF that we estimate.
Our estimate is based on the {\sf Count-Min} sketch, 
see \cite{Cormode:Muthu:2005},
 extended to 
continuous random variables.  This is based on the 
probabilistic selection
of hash functions.  The reason to use the estimated 
EDF (EEDF) is that
we want to be sure that we can shrink the interval 
length in the histogram generated from
the {\sf Count-Min } sketch so small that it gives
an arbitrarily good approximation of the EDF we use
and hence DF if it exists.
Our hash-based predictor (HBP) is effectively a 
mean of the EEDF. 
In our computational work, we use multiple versions 
of our {\sf Count-Min} sketch based predictor.

Our {\sf Count-Min} based predictor will outperform 
the usual EDF predictor when the sample
size and number of items in the stream
is very large.  This is partially due
to the presence of
memory constraints.  Briefly, the memory 
needed by the {\sf Count-Min}
sketch can be chosen by the user.  More memory
produces better quality predictors and 
there are
formal results ensuring the estimators 
are close
to the true value when it exists.  In fact,
by construction, estimators from the 
{\sf Count-Min}
sketch never underestimate the true frequencies of
elements so it favors high frequency elements even
though low frequency elements maybe overestimate
albeit not by much.  In essence, using the
{\sf Count-Min} sketch gives a sort of data compression that we hope preserves only the
predictively important aspects of the data.

The second predictor we propose is based on 
Gaussian processes that have a random additive
bias.
It has long been
known that the posterior distribution can be regarded as
an input-output relation giving a distribution for a
specific data set as if deterministically, see 
\cite{Chen:1985}, Sec. 3.  This means Bayesian
predictors can legitimately be used.  On the other hand, 
in ${\cal{M}}$-open problems, we may not be able to 
identify useful properties of the data generator.  
So, we want to prevent posterior distributions from 
converging and thereby misleading us into
believing their limit.
Modifying a Gaussian process to include a random
bias helps ensure that unwanted convergence won't
occur.  

There are a variety of other existing techniques for this
class of problem.  Perhaps one of the earliest is
the Shtarkov solution, see \cite{Shtarkov:1988}),
sometimes called the normalized maximum likelihood.
The Shtarkov solution is based on log-loss and requires
the analyst to choose a collection of `experts', 
essentially parametric families, and
tries to match the performance of the best of them.  
Different Shtarkov solutions result from
different choices of experts.   Computational 
and theoretical work
on the Shtarkov solution is extensive and often from
a very general perspective, see \cite{Barron:etal:2014}
and \cite{Yang:Barron:2017}.  Moreover,
the log-loss is commonly used to define the concept
of regret, see \cite{Xie:Barron:2000}.
Here, we use a Bayesian version of the Shtarkov
solution that is easier to work with see 
\cite{Kontkanen:Myllymaki:2007} and \cite{Le:Clarke:2016}.
The specific form of Shtarkov predictor 
we use here is based on normality and is
a ratio of Shtarkov solutions.  Thus, it
{\em mimics} a conditional density.    
It is not strictly a conditional
density because the Shtarkov solution does not
satisfy the marginalization requirement for
stochastic processes.  However,
the mode of the Shtarkov predictor often performs well.

In addition, in our computational comparisons, we 
include two other well-known Bayesian predictors, one
based on regular Gaussian processes and one 
based on Dirichlet processes. 
Being Bayesian,
both of these require prior selection and when needed
we use an empirical Bayes approach.
In general, Bayesian methods
assume a stochastic model for the data and are expected to
perform best when the model is approximately true but 
poorly otherwise.

One of the important features of sketches is that they 
can be implemented in one pass.  So, for the sake of 
completeness we compare one-pass versions of our
predictors when they exist
to versions of our predictors formed
by reducing past data to a representative subset.
We do this using streaming $K$-means but any streaming
clustering algorithm that can give cluster centers
could be used.  In simple cases such as those here where
there are no explanatory variables and $y$
is unidimensional the choice of clustering procedure should
not make much difference.

Other techniques that can in principle be applied 
to $\mathcal{M}$-open data include various forms
of conformal prediction
\cite{Vovk:Shafer:2008, Barber:etal:2023}. 
However, 
it is not clear how to mesh this approach with 
$\mathcal{M}$-open data when 
the defining feature of conformal techniques is that
future data `conforms' to past data.
The same criticism applies to time series and other
online techniques,
again with the caveat they could be used as long as they
are regarded strictly as input-output relations
rather than having any necessary relation to
the data generator which is ${\cal{M}}$-open.

Also, Bayesian techniques (especially
nonparametrics \cite{Ghosal:Vandervaart:2017}), 
score 
function
based techniques (see \cite{Gneiting:2011}
\cite{Dawid:Musio:Ventura:2016}), and conventional
frequentist point
prediction can be used in the $\mathcal{M}$-open context 
provided these techniques are regarded strictly 
as input-output
relations or actions, i.e., as not having any modeling 
validity.  Apart from DPP's and GPP's, 
we have not included 
more of these various methods for 
space limitations
because they are so numerous.

As suggested by our focus on ${\cal{M}}$-open 
predictive problems, methods designed for 
streaming data that cannot be effectively
stochastically modeled tend to perform better
than methods that aren't.  Accordingly,
our computational work is consistent with
the proposal that there
is some sort of matching principle between 
the complexity
of the data and the complexity of an
optimal or at least good predictor.
Among the predictors we develop and study,
one version of our HBP predictor 
(one pass, median based) has the best 
performance
in a cumulative $L^1$ sense for the 
most complex data
we use.  Second place among the methods we
studied on the data sets we used goes to
a predictor based 
on Gaussian processes with a random bias.
Shtarkov predictors can perform 
roughly comparably in some cases but are
usually worse; we think this is so because the
Shatarkov predictor
we use here is the simplest
possible.
Theoretical comparisons of these
methods in the context that interests us
are difficult because they represent three
conceptually different class of predictors.

In the next three sections we formally 
present the three classes of
methods we study here.  
All are applicable in streaming data settings 
where no stochastic model can be assumed.
In Sec. \ref{HBPmethod}, we 
define HBP methods and give various properties of
them including a sort of consistency, space bound,
and `classical' convergence properties.
In Sec. \ref{GPPs}, we present our Bayesian
predictors. We define standard Gaussian
process prior (GPP) predictors and extend them to the
case that a random additive bias term is included
enabling us to derive a new class of GPP
based predictors. 
We also define Dirichlet process prior point prediction.
In Sec. \ref{Shtarkov}, we present our Shtarkov based 
predictors, based on the normalized maximum likelihood
in the normal case where explicit expressions can be derived.
Then in Sec. \ref{compresults} we present our computational comparisons.
We conclude in Sec. \ref{discussion} with some general observations.

\section{Hash Function Based Predictors}
\label{HBPmethod}

We adapt the Count-Min sketch algorithm so it can be used with real data to estimate an empirical distribution function. The idea is to partition the real data into intervals of equal size and then compute the relative frequencies of the intervals as an approximation of the distribution function.  We call techniques based
on hash functions hash based predictors (HBP's).
Once we define our predictors, we establish
some of their properties while preserving the
${\cal{M}}$-open nature of the data.

\subsection{The HBP Method}

For the domain $[K] = \{1,2,\cdots, K\}$ and the range $[W] = \{1,2,\cdots,W\}$, let 
Let $\mathcal{H} \subseteq \{h\mid h: [K]\longrightarrow [W]\}$ 
be a class of hash functions. 
Ideally, we would like $\mathcal{H}$ to be entire class 
of functions. However, in computational settings, 
due to space limitations, we assume $\mathcal{H}$ 
to be a $2$-universal family.  This means that 
for any $x_1, x_2 \in [K]$ with $x_1 \neq x_2$ 
and $y_1,y_2 \in [W]$, 
$P_\mathcal{H} (H(x_1) = y_1 \cap H(x_2) = y_2) 
= {1\over N^2}$, where $H$ is the random variable varying
over $\mathcal{H}$ with probability $P_\mathcal{H}$
uniform over $h \in \mathcal{H}$.  

Now we describe our predictor. We assume that the stream $(y_1,y_2,\ldots, y_n,\ldots)$ are real numbers from the range $(m,M]$ for some real $m$ and $M$. Fix $K \in \mathbb{N}$. Partition $(m,M]$ into $K$ intervals each of length $\frac{M-m}{K}$. Further, for the rest of the discussion, without loss of generality, we let $m = 0$. Denote the $k^{th}$ interval by $I_{Kk}$, for $k= 1,2,\cdots,K$. That is, 
\begin{equation}
\label{disjointinterval}
  I_k = I_{Kk}  = \left((k-1)\frac{M}{K}, k\frac{M}{K} \right].  
\end{equation}
The goal is to predict $y^{n+1}$ after seeing $(y_1,\cdots,y_n)$. 
Let 
$$
a_k = a_{Kk}(n) = \#\{y_i \in I_{k} \mid i=1,\cdots,n\}.
$$
That is, $a_{k}(n)$ is the frequency of items in the stream $(y_1,\cdots,y_n)$ that fall in $I_{k}$. 

Let $h_1,\cdots,h_{d_{K}}$ be $d_K$
randomly chosen hash functions where the domain is $[K]$ 
and the range is [$W_K$].  That is, 
$\forall j=1,\cdots,d_K$; $h_j: \{1,\cdots,K\} \longrightarrow \{1,\cdots,W_K\}$. 
Here $d_K$ and $W_K$ are parameters that can 
be chosen by the 
user.  For effective data compression, or space efficiency, 
typically $W_K << K$. 
Next, we extend our hash functions to the continuous domain 
$(m, M]$ by defining
$$
\Tilde{h}_j: (m,M] \longrightarrow \{1,2,\cdots,W_K\}
$$
where $\Tilde{h}_j(s) = h_j(k)$ for $s \in I_{k}$. Note that when $i \leq n$, $y_i \in  I_{k}$ gives
$$
\Tilde{h}_j(y_i)=h_j(k).
$$

Denote the number of times the $j$-th hash function 
makes an error, i.e., 
$h_j$ assigns the same value to two different elements $k, \ell$ of its domain $I_{k}$ by
\begin{eqnarray}
\label{discrete}
   I_{kjl}= 
\begin{cases}
    1,& \text{if } {h_j(k) = h_j(l); k\neq l} \\
    0,              & \text{otherwise}  .
\end{cases} 
\end{eqnarray}
We tolerate the (small) error for the sake of 
compression. Next, we extend 
\eqref{discrete} to the interval case by writing
\begin{eqnarray}
\label{interval}
 I_{s_1,j,s_2}= 
\begin{cases}
    1,& \text{if } {\tilde{h}_j(s_1) = \tilde{h}_j(s_2); s_1 \nsim s_2 } \\
    0,              & \text{otherwise},
\end{cases}
\end{eqnarray}
where $\nsim$ means that $s_1$ and $s_2$ are in disjoint sub-intervals of (m,M].
We link the extent to which $h_j$ is not one-to-one with the
occurrence of $y_i$ in the intervals by defining
\begin{eqnarray}
\label{link}
X_{kj}(n) = \sum_{l=1}^{k} I_{s_1,j,s_2}a_{l}(n).
\end{eqnarray}
In \eqref{link}, as in \eqref{interval}, we think of $s_1 \in I_{k}$ and $s_2 \in I_{l}$ and note
$X_{kj} \geq 0$.  More precisely, 
we see that $X_{kj}(n)$ is the number of elements e.g., $y_i$'s, in the stream up to time $n$  
that are not in $I_{k}$ but still give $h_j(k)$, i.e., $\tilde{h}_j(y_i)={h}_j(k)$.

We next define an estimate of  ${a}_{k}$ (frequency of the $k^{th}$ interval) denoted by 
$\hat{a}_{k}$, at time $n$. For the $j^{th}$ hash function $h_j$,  an interval $k$ and time $n$, define ${\sf count}_n(j,h_j(k))$ as follows:
\begin{eqnarray}
{\sf count}_n(j,h_j(k)) = 
\# \{i \leq n |~ \tilde{h}_j(y_i) = h_j(k)\}.
\nonumber
\end{eqnarray} 
For the $j^{th}$ hash function let $\hat{a}_{jk}(n) = {\sf count}_n(j,h_j(k))$. 
Then the estimate $\hat{a}_{k}$ is defined as  
\begin{eqnarray}
\hat{a}_{k}(n) = \min_{j} ~ \hat{a}_{jk}(n) \geq 0.
\label{freqest}
\end{eqnarray}

To find the next prediction, we use two methods.
These are essentially weighted means and medians
which we define for the sake of being explicit.
Given $y^n$, the predicted value for $y_{n+1}$ is:
\begin{itemize}
    \item the weighted mean of the midpoints
    of the intervals $I_{k}; k = 1,2, \cdots, K$ 
    defined in \eqref{disjointinterval}, where the weights 
    are $\hat{a}_{k}$ as defined in \eqref{freqest}. 
    Let the mid-point interval $I_k$ be $m_k$. Then, 
\begin{eqnarray}
\label{CMpredictionmean}
    \hat{y}_{n+1} = \hat{y}_{K, n+1} =  \sum_{k = 1}^{K} m_k\frac{\hat{a}_{k}(n)}{n};
\end{eqnarray}
\item the weighted median of $m_k$ with weights 
$W_K = \frac{\hat{a}_{k}}{\sum_{k=1}^{K}\hat{a}_{k}}$ 
is defined as the average of $m_q$ and $m_{q+1}$, 
where $m_q$ satisfies
\begin{eqnarray}
\label{CMpredictionmed}
    \sum_{k=1}^{q} W_k \leq \frac{1}{2}
\text{ and }
\sum_{k=q+1}^{K} W_k \geq \frac{1}{2} .
\end{eqnarray}
\end{itemize}

\subsection{A Few Key Properties}

Here we prove several important properties of our
use of the {\sf Count-Min} sketch to estimate
the EDF.  The first is a bound on $\hat{a}_{jk}(n)$
in terms of $a_k(n)$.  Part of the novelty here is
that the mode of convergence
is defined by $P_\mathcal{H}$, a distribution on
the hash functions not on the data. When we write $H$ we mean the random variable distributed according to $P_\mathcal{H}$ that assumes values $h$.
This allows us to preserve the assumption that 
the $y^n$ does not have a distribution and hence
remains ${\cal{M}}$-open.  That is, we treat the
$y_i$ as if they were real numbers with no stochastic
properties.  We indicate clearly below the few cases where
we depart from this to show `counterfactual' results.

\subsubsection{Bounds on error and storage}
\label{bounds}

Our first theorem is that $\hat{a}_{k}$ estimates the frequency of an interval $I_{k}$ well, asymptotically.  
Let $||a||_1= \sum_{k=1}^{K} a_{k}(n) = n$ be the sum over $k$ of the number of elements $y_i$ up to time $n$ that land in $I_{k}$ ($K$ and $n$ are suppressed in the notation $||a||_1$ for brevity).
Our first result -- for fixed $K$ -- is a consistency result for \eqref{freqest}.  After that, 
we provide a `space bound' on the storage requirement
for the sketch.   Both are similar to the guarantees
for the Count-Min sketch.

\begin{theorem}
\label{upperbd}
 $ \forall \epsilon > 0 ~ \forall \delta > 0 : \exists \mbox{ N } ~  \forall d_K > N$ such that  $P(\forall j = 1,\cdots,d_K; \hat{a}_{jk}(n) \leq a_{k}(n)+\epsilon||a||_1 ) \leq \delta$.
\end{theorem}

{\bf Remark:}  As needed, we use the fact that $a_k(n)  \leq \hat{a}_k(n)$ for any $n$, by construction, without further
comment.

\begin{proof}  For ease of readability, we break the
proof into three steps.
\begin{flalign*}
& \begin{aligned}
\textbf{Step 1: } E(I_{kjl}) &= 1.P\left(h_j(k)=h_j(l)\right) \\[0.1in]
                             &= \sum_{v=0}^{W_K}P\left(h_j(k) =v= h_j(l)\right)\\[0.05in] 
                             &\stackrel{\text{indep}}{=} \sum_{v=0}^{W_K}P\left(h_j(k)= v\right) P\left(h_j(l)=v\right)\\[0.05in]
                             &\stackrel{\text{unif}}{=} \sum_{v=0}^{W_K} \frac{1}{W_K+1}\frac{1}{W_K+1}
                             = \frac{W_K+1}{(W_K+1)^2} = \frac{1}{W_K+1} < \frac{1}{W_K} \leq \epsilon/e.
 \end{aligned}
   &
\end{flalign*}
\begin{flalign*}
& \begin{aligned}
 \textbf{Step 2: } E(X_{kj}(n)) &= E\left(\sum_{k=1}^{K} I_{kjl}a_{l}(n)\right)\\[0.05in] 
&= \sum_{k=1}^{K}a_{l}E(I_{kjl})\\[0.05in] 
&< \sum_{k=1}^{k} a_{l}\epsilon/e = ||a||_1\epsilon/e.
\end{aligned}
   &
\end{flalign*}
\begin{flalign*}
& \begin{aligned} \textbf{Step 3: }  
P\left(\hat{a}_{jk}(n) > a_{k}(n)+\epsilon||a||_1\right)
&= P\left(\cancel{a_{k}(n)}+X_{kj}(n)>\cancel{a_{k}(n)}+\epsilon||a||_1\right)\\[0.1in]
&= \left(P(X_{kj}(n)>eE(X_{kj}(n))\right).[Since, eE(X_{kj})<\epsilon||a||_1]\\[0.05in]  
&\leq \frac{1}{eE(X_{kj}(n))}E(X_{kj}(n)) = 1/e .\\
\end{aligned}
   &
\end{flalign*}  

Since, the $\hat{a}_{kj}(n)'s$ are independent (their hash functions are independent), if we set $-\log\delta = d_K$, we have
\begin{flalign*}
& \begin{aligned}P(\forall j; \hat{a}_{kj}(n) > a_{k}(n)+\epsilon||a||_1) 
\leq (1/e)^{d_K} < (1/e)^{\log(1/\delta)} =\delta.  \quad \square
 \end{aligned}
\end{flalign*}

\end{proof}

{\bf Remark}: Here, $\|a \|_1 = n$ because we are looking at data streams in the `cash register' model, i.e., items only accumulate.   We conjecture extensions to the turnstile model
can be given even if they do not seem relevant here.


Next, we address the storage requirement
for the procedure used in
Theorem \ref{upperbd}.  
Heuristically, observe that the storage is upper 
bounded by 
the number of hash functions 
$\log(1/\delta)$ multiplied by the
number of values each hash function can take, 
namely $e/\epsilon$ or ${\cal{O}}(1/\epsilon)$
giving ${\cal{O}}((1/\epsilon)\log(1/\delta))$.
Following \cite{Muthukrishnan:2009}, we see in the following that
${\cal{O}}(1/\epsilon)$ will suffice.

\begin{theorem}
\label{storage}
Let $\delta > 0$. If the storage available is  $\Omega(1/\epsilon)$\footnotemark, then,
\begin{eqnarray*}
    P(\hat{a}_{jk} \leq a_{k}+\epsilon||a||_1) \leq \delta
\end{eqnarray*}
\end{theorem}

\footnotetext{$\Omega$-notation gives a lower bound
in contrast to big-${\cal{O}}$ notation that gives an
upper bound.}

\begin{proof}
From \cite{Muthukrishnan:2009}, we get the proof for the space requirement for estimating the frequency $a_k$ of the distinct elements $k$ of the stream with an error of $\epsilon||a||_1$. Since, our set up is different
from \cite{Muthukrishnan:2009} in that we are looking at
continuous data rather than discrete data we rewrite the distinct elements $1,2,\cdots,K$ 
in \cite{Muthukrishnan:2009}. We write the  
frequency vector $a = a_k(n)= (a_1(n), a_2(n), \cdots, a_K(n))$ as the 
frequencies of the K intervals. The estimated frequency in \cite{Muthukrishnan:2009} can
be rewritten as $\hat{a}_{k}$ for any interval $k$. 
Now, $||a||_1$, which is the sum of all frequencies of 
the $K$ intervals. Choosing 
$K = 1/(2\epsilon)$ in \cite{Muthukrishnan:2009}, 
we recall that one needs to 
use $\Omega(1/\epsilon)$ unit of space to obtain
\begin{eqnarray*}
  P(\hat{a}_{jk} \leq a_{k}+\epsilon||a||_1)   \leq \Big(\frac{1}{2}\Big)^{log(\frac{1}{\delta})} < \Big(\frac{1}{e}\Big)^{log(\frac{1}{\delta})} = \delta ;
\end{eqnarray*}
  see \cite{Muthukrishnan:2009} for details.
\end{proof}

\subsubsection{Convergence of the EEDF in probability}

We extend Theorem \ref{upperbd} by 
letting $K, d_K,  n \rightarrow \infty$ at appropriate rates to get a consistency result
for the `estimated' EDF (EEDF).  That is, our EEDF converges to an EDF\footnotemark~based
on the streaming data that is not necessarily the true DF since it needn't exist.  This is the content
of our next result.

\footnotetext{Strictly speaking, our $F_n$ is not the
usual empirical distribution function.  It is the estimator of a distribution function based on a 
histogram estimator.  We do this for the sake of readability.}
\begin{theorem}
\label{EEDFtoEDF}
Let $y_i \in (m,M]$.  Then, pointwise in $y_i$,
 \begin{align*}
     \hat{F}_{n}(y_i)-F_{n}(y_i)\overset{P}{\to} 0 \mbox{ as }d_K, K, \mbox{and } n \longrightarrow \infty.
 \end{align*} 
 where the EEDF is $\hat{F}_{n}(y_i) = \sum_{k\leq y_i}^{} \frac{\hat{a}_{k}(n)}{n}$ and
the EDF is $F_{n}(y_i) = \sum_{k\leq y_i}^{} \frac{a_{k}(n)}{n}$.
\end{theorem}

\begin{proof}
     We have
 \begin{align*}
 \mathbb{E}\left(\frac{\hat{a}_{k}(n)}{n}\right)&=\frac{1}{n}\mathbb{E}\min_{j=1}^{d_K} count_{n}(j,h_j(k))
\\
&= \frac{a_{k}(n)}{n}+\frac{1}{n}\mathbb{E}\left[\min_{j=1}^{d_K}\bigg\{\sum_{l=1}^{K}I_{k,j,l}a_{l}(n)\bigg\}\right] \\
&= \frac{a_{k}(n)}{n}+\Gamma,
\end{align*}
where $\Gamma = \frac{1}{n}\mathbb{E}\left[\min_{j=1}^{d_K}\bigg\{\sum_{l=1}^{K}I_{k,j,l}a_{l}(n)\bigg\}\right]$.  Setting $j=1$ gives an upper bound.  So, for any
$K$, we have
\begin{align*}
     \Gamma &\leq \frac{1}{n}\mathbb{E}\left[\sum_{l=1}^{K}I_{k,1,l}a_{l}(n)\right]\\
   & =\sum_{l=1}^{K}\frac{a_{l}(n)}{n}\mathbb{E}I_{k,1,l} \\ 
   & \leq \sum_{l=1}^{K}\left(\frac{1}{W_K}\right)\frac{a_{l}(n)}{n} \\ 
&= \frac{1}{W_K}\sum_{l=1}^{K}\frac{a_{l}(n)}{n} 
   = \frac{1}{W_K}.   
 \end{align*}
 Letting $W_K \rightarrow \infty$ gives that the RHS goes to zero.
Thus, as $W_K$ increases,
 \begin{align*}
      \mathbb{E}\left(\frac{\hat{a}_{k}(n)}{n}\right) - \frac{a_{k}(n)}{n} \longrightarrow 0
\end{align*}
and for each $k$
\begin{align*}
      \frac{\hat{a}_{k}(n)}{n} - \frac{a_{k}(n)}{n} \overset{P}{\to} 0.
\end{align*}
    Hence, 
\begin{align*}
\sum_{k\leq y_i}^{} \frac{\hat{a}_{k}(n)}{n} - \sum_{k\leq y_i}^{}\frac{a_{k}(n)}{n} \overset{P}{\to} 0,
\quad i.e.,  \hat{F}_{n}(y_i) -  F_{n}(y_i) \overset{P}{\to} 0.\quad \square
 \end{align*} 
\end{proof}

Next, if the EDF based on streaming data actually has a limiting DF then the
EEDF and the EDF itself converge to $F$. 
For this Corollary, we are 
counterfactually
assuming the data stream is 
${\cal{M}}$-complete or -closed.
Our result is the following.

\medskip

 \begin{t2}
\label{cor1}
If $y_i \sim F$ independently and identically
for some DF $F$ whose density has 
a bounded derivative, 
then for any 
$y_i \in [m, M]$, $F_n(y_i) \longrightarrow F(y_i)$ 
and
$\hat{F}_n(y_i) \rightarrow F(y_i)$
in probability.
\end{t2}

\medskip

{\bf Remark:} These statements are included for
the sake of confirming intuition and can be greatly
generalized.  In fact, since our application is to
${\cal{M}}$-open data, we can never make
stochastic assumptions.

 \begin{proof}
The estimator $F_n$ is the EDF from a histogram estimator
and does not depend on hash functions -- only
on $y^n$ and $K$.  Under the stated
conditions, the histogram estimator converges to the
density of $F$ in probability and hence $\hat{F}_n(y_i) \rightarrow F(y_i)$ in the probability defined by $F$.  
The second statement follows from using Theorem
\ref{EEDFtoEDF} and the triangle 
inequality, adding and subtracting $F_n$.  The mode of convergence for the
terms is different; one is $P_{\cal{H}}$ and the other
is $F$.
\end{proof}

\subsubsection{A streaming Glivenko-Cantelli theorem.}
\label{streamingGC}

A version of the standard 
Glivenko-Cantelli Theorem can be found in \cite{Chung:1974}. The structure of our proof below
is based heavily on \cite{Shaikh:2009}.
Our result for streaming data 
goes beyond this by assessing convergence
in the joint distribution of the hash functions 
and the data. 
So, this result, like the last, must be interpreted
counterfactually.  For any DF $G$ we write $G(y_i^-)$
to mean its limit from the left at $y_i$.

\medskip

\begin{theorem}
\label{GCtheorem}
Suppose the $y_i$'s are independently and 
identically distributed
according to $F$ and let $y_i \in (0,M]$.
Then, as $n \rightarrow \infty$, there are rates
$K_n, d_{K_n}, ~\mbox{and}~ W_{K_n} \rightarrow \infty$ so that 
the following
Clauses are satisfied.

\begin{description}
\item 
\label{clause1gc}
Clause I:
    If $F_{n}(y_i) \overset{P}{\to} F(y_i)$ and $F_{n}(y_i^-) \overset{P}{\to} F(y_i^-)$ pointwise for all $y_i \in (0,M]$, then, $\underset{y_i \in \mathbb{R}}{Sup} |F_{n}(y_i)-F(y_i)| \longrightarrow 0$ in probability. 
    
\item Clause II: If Clause I holds, then $\underset{y_i \in \mathbb{R}}{Sup} |\hat{F}_{n}(y_i)-F(y_i)| \overset{a.s.}{\to} 0$.

\item 
Clause III:  The EEDF converges to the EDF, i.e.,
$\underset{y_i \in \mathbb{R}}{Sup} |\hat{F}_{n}(y_i)-F_{n}(y_i)| \overset{a.s.}{\to} 0$ .

\end{description}
\end{theorem}

\noindent{\bf Remarks:}  Clause I is the usual Glivenko-Cantelli theorem but formally for $F_n$ rather
than for the usual EDF.  Clause II is for the
CMS-generated EEDF to converge to $F$.
Because it involves only empirical quantities,
Clause III is what we really want, namely the
convergence of the EEDF to the EDF in a mode
stronger than that used in Cor. \ref{cor1}.
Again, we only show the simplest (IID) case since
our point is only to develop the heuristics.
These results can be greatly generalized to
include many dependent data settings, although 
this is not important for the use of our
hash based predictor:  our interest is
in the ${\cal{M}}$-open case where such
assumptions are irrelevant.

\begin{proof}
A complete proof is given in the Appendix, 
Sec. \ref{GPderive}.
\end{proof}

It is essential to remember that
the randomness in $\hat{F}_{n}(y_i)$ 
does not come from the data points $y$ except when we
used a distribution on $y$ to prove results such
as those above.  In fact, the randomness in 
$\hat{F}_{n}(y_i)$ comes from the hash functions
via the $\hat{a}_{k}$'s.
One of the points of Theorem \ref{EEDFtoEDF} or
Theorem \ref{GCtheorem} is that in principle
we can 
obtain asymptotically valid prediction intervals,
not just point predictors, 
from an EDF or EEDF, at least in the
${\cal{M}}$-closed and -complete cases.

A useful property of the EEDF is that it can
track the location of the data. For example, if we 
have an initial set of $n$ data points that follow
a $N(0,1)$ distribution, the EEDF for these points 
will look like a $N(0,1)$.  However, if later points
follow a $N(1,1)$ distribution, as they accumulate
the EEDF will shift from $N(0,1)$ to $N(1,1)$. 
The EEDF is adaptive in that it
can reconverge to a new distribution.

\section{Bayesian Predictors}
\label{GPPs}

In this section we define three Bayesian predictors.
The first is the usual Gaussian Process Prior
predictors.  The second is an extension of this
to include a random additive bias.   The third is
the usual Dirichlet Process Prior predictor,
essentially the Bayesian's histogram
possibly mimicking the EDF or EEDF.  Predictive
distributions are well-known for the first and
third of these; we review them here for the
sake of completeness.    We provide full details
for the second since it seems to be new.
Recall that these must be seen as predictors
only; the data being ${\cal{M}}$-open means
that modeling e.g., by the convergence of a
Bayes model, would be a contradiction.

\subsection{No Bias}
\label{GPPRV}

We assume  $Y_i = f_i+\epsilon_i , \quad i = 1,\cdots, n$
where the $i^{th}$ data point $y_i$ is distributed 
according to $Y_i$ and $f = (f_1,f_2,\cdots,f_n)^{T}$ 
is equipped with a Gaussian process prior.  
That is, $f \sim \mathcal{N}(a,\sigma^2K_{11})$, 
where, $a = (a_1,a_2,\cdots,a_n)^{T}$ is the mean and
$K_{11} = \biggl(\Bigl(k_{ij}\Bigl)\biggl) ; i,j = 1, \cdots, n $ is the covariance function in which $k_{ij} = k_{ij}(y_i, y_j)$.
First, we assume there is no bias i.e.,
$a_i = 0$ for all $i$, so
the joint distribution of 
$Y = (Y_1,Y_2,\cdots,Y_n)^T$ and $Y_{n+1}$ is
\begin{eqnarray} 
 \begin{pmatrix}Y\\
\vdots\\
Y_{n+1}
\end{pmatrix} 
&=& 
\begin{pmatrix}f\\
\vdots\\
f_{n+1} \end{pmatrix} 
+ 
\begin{pmatrix}\epsilon_1\\ 
\vdots\\
\epsilon_{n+1} \end{pmatrix} 
\nonumber\\
& \sim & 
\mathcal{N}\left[\left(\begin{array}{c}
0\\
\vdots\\
0
\end{array}\right),\sigma^2\left(\begin{array}{ccc}
K_{11}+I & \vdots &K_{12}\\
\cdots & \vdots & \cdots\\
K_{21} & \vdots & K_{22}+1 
\end{array}\right)\right] 
\label{ynplus1distnobias}
\end{eqnarray}
where 
$K_{12} = (k_{1,n+1}, k_{2,n+1}, \cdots, k_{n,n})^T$
and
$K_{21} = K_{12}^T$, 
$K_{22} = k_{n+1,n+1}$.
More compactly, we write 
\begin{eqnarray}
\label{GPPnobias}
    Y^{n+1} \sim \mathcal{N}(0^{n+1}, \sigma^2(I+K)_{n+1 \times n+1}).
\end{eqnarray}

It is well known that the predictive distribution of $Y_{n+1}$ given $Y$ is
\begin{eqnarray}
Y_{n+1}|Y &\sim& \mathcal{N} (\mu^*, \Sigma^*)
\nonumber
\end{eqnarray}
where
\begin{eqnarray}
\mu^* &=& \sigma^2K_{12} \{\sigma^2(K_{11}+I)\}^{-1}y = K_{12} \{(K_{11}+I)\}^{-1}y 
\label{mustarNRB}
\end{eqnarray}
and 
\begin{eqnarray}
\Sigma^* &=& \sigma^2(K_{22}+1)-K_{21}\{\sigma^2(K_{11}+I)\}^{-1}\sigma^2K_{12} \nonumber\\
&=&  \sigma^2(K_{22}+1)-K_{21}(K_{11}+I)^{-1}K_{12}
\label{vceNRB}
\end{eqnarray}
Hence, in the zero bias case, our optimal point 
predictor (under squared error loss for instance) is
simply the conditional mean $\mu^*$ 
in \eqref{mustarNRB}.  

To complete the specification
it remains to estimate $\sigma^2$ for use in \eqref{vceNRB}. In the general
case, we have $Y \sim \mathcal{N}(a,\sigma^2(I+K)_{n \times n})$. 
Hence, $(I+K)^{\frac{1}{2}}Y \sim \mathcal{N}(a,\sigma^2I)$.  
Letting $Y' = (I+K)^{\frac{1}{2}}Y$
and $S_k = \frac{1}{n-1} \sum_{i=1}^{n}(y_i'-\bar{y'})^k$
we can estimate $\sigma^2$ by $S_2$.
Note that $\sigma^2$ cancels out in \eqref{mustarNRB} and since we are only looking
at point prediction in our computations below
we do not have to use \eqref{vceNRB}.

\subsection{Random Additive Bias}
\label{GPPRB}

Consider a Gaussian process prior in which the bias 
$a = (a_1, \ldots,  a_n)^T$ is random.  That is, when
we estimate function value -- an $f_i$ for $i \leq n$ --
the prior adds a small amount of bias effectively
enlarging the range of the estimate.  For the
prediction of $f_{n+1}$ a similar sort of widening
happens.  To see this, write
\begin{align}
    a \sim \mathcal{N}(\gamma {\bf 1}_n,\sigma^2\delta^2I_{n \times n})
    \label{pdfa}
\end{align}
where the expected bias is $\gamma \in \mathbb{R}$ and 
$\sigma^2 > 0$ has distribution
\begin{align}
     \sigma^2 \sim \mathcal{IG}(\alpha,\beta) .
     \label{sigdist}
\end{align}
Here, $\alpha$, $\beta$, and $\delta$ are 
strictly positive, and, like $\gamma$ are unknown.
Expression \eqref{pdfa} means that,
with some loss of generality, the
biases are independent, identical, 
symmetric, unimodal, 
and have light tails.  
Since
\begin{eqnarray}
    Y \sim \mathcal{N}(a,\sigma^2(I_{n \times n}+K_{n \times n})),
    \label{conditionalfory}
\end{eqnarray}
its likelihood is
\begin{eqnarray}
\mathcal{L}_1(a,\sigma^2|y) 
    &=& \mathcal{N}(a,\sigma^2(I_{n \times n}+K_{n \times n}))(y) \nonumber \\
  &=&
  \frac{e^{-\frac{1}{2\sigma^2}(y-a)'(I_{n \times n}+K_{n \times n})^{-1}(y-a)}}{(2\pi)^\frac{n}{2}(\sigma^2)^\frac{n}{2}|I_{n \times n}+K_{n \times n}|^\frac{1}{2}}  
  \label{likelihoodeq}
 \end{eqnarray}
and the joint prior for $(a,\sigma^2)$ is 
\begin{eqnarray}
    w(a,\sigma^2) &=& \mathcal{N}(\gamma1,\sigma^2\delta^2 I_{n \times n})\textbf{ }\mathcal{IG}(\alpha,\beta) \nonumber \\
        &=&\frac{e^{-\frac{1}{2\sigma^2}(a-\gamma1)'(\delta^2I_{n \times n})^{-1}(a-\gamma1)}}{(2\pi)^\frac{n}{2}(\sigma^2\delta^2)^\frac{n}{2}}\frac{\beta^\alpha}{\Gamma(\alpha)}\times \left(\frac{1}{\sigma^2}\right)^{\alpha+1}e^{-\frac{\beta}{\sigma^2}} .
    \end{eqnarray}

Our first result is the identification of the posterior
predictive density for $Y_{n+1}$ given $y^n$.  We
have the following.

\begin{theorem}
\label{studentt}
The posterior predictive distribution of the future observation $y_{n+1}$ given the past observations 
$y^n$ is 
\begin{align}
    m(y_{n+1}|y^n) = St_{2\alpha+n}\bigg(A_1, \frac{\beta^{**}}{\frac{2\alpha+n}{2}}\bigg)(y_{n+1}),
    \label{tparameters}
\end{align} 
where $S\text{t}_v(\theta,\Sigma)$ denotes the 
Student's $t$ distribution with $v$ degrees 
of freedom with parameters $\theta$ and $\Sigma$. 
In \eqref{tparameters}, $\beta^{**} = \beta + A_2 $ and
$A_1 = \frac{\gamma_2-y'^ng_1^n}{\gamma_1}$.
Expressions for
$g_1^n$, $\gamma_1$, $\gamma_2$, and $A_2$
are given in the proof and
can be explicitly written as functions of the 
variance matrix $K_{n +1 \times n+1}$, $y^n$, $\gamma$, and $\delta$.
\end{theorem}

{\bf Remark:} Estimation of $\gamma$ and $\delta$
will be discussed after the statement of Theorem
\ref{gammadelta} has been given below.

\begin{proof}
A complete proof is given in the Appendix, Sec. \ref{GPderive}.
\end{proof}

To use this result, we must have a way to obtain values
for the hyperparameters $\alpha$, $\beta$, and $\delta$
and for the bias $\gamma$.
Starting with $\alpha$ and $\beta$,
recall \eqref{sigdist} and
define 
$S_k = \frac{1}{n-1} \sum_{i=1}^{n}(y_i'-\bar{y'})^k$.
For an inverse gamma we have
 \begin{eqnarray}
\label{expect}
   E(S_2) &=& \frac{\beta}{\alpha-1}\\
   Var(S_2) &=& \frac{\beta^2}{(\alpha-1)^2(\alpha-2)}.
   \label{var}
 \end{eqnarray}
 Now, we can solve solve for $\alpha$ and $\beta$ from \eqref{expect} and \eqref{var} and invoke the method of moments to find
 \begin{eqnarray}
 \label{alphaest}
   \hat{\alpha} \approx \frac{S_2}{S_4 - S_2^2} + 2 \\
   \hat{\beta} \approx S_2(\hat{\alpha}+1),
   \label{betaest}
 \end{eqnarray}
where we have used the same estimate $S_2$
of $\sigma^2$ as in Subsec. \ref{GPPRV}. 

To estimate the parameters $\gamma$ and $\delta^2$
we form the likelihood $\mathcal{L}_3(y \vert \gamma, \delta^2, \sigma^2)$ by integrating out $a$ from the
product of \eqref{pdfa} and \eqref{likelihoodeq}
and maximize it.  To do this,
we state a result that gives the forms of the likelihood we want to maximise, writing the same likelihood in two different ways so the optimization will be clear.
We also use this result to estimate the 
parameters in the location
 and scale of the predictive distribution in 
 Theorem \ref{studentt}.
 
\begin{theorem}
\label{gammadelta}
The likelihood of $y^n$ given $\gamma, \delta^2$ and $\sigma^2$, marginalizing out
$a$, can be written in two equivalent forms:

\noindent
Clause I:
\begin{eqnarray}
 \mathcal{L}_2(y^n|\gamma, \delta^2, \sigma^2) &=& h(\gamma) \frac{|V_{n \times n}|^{\frac{1}{2}}}{(2\pi\sigma^2\delta^2)^{\frac{n}{2}}|(I+K)_{n \times n}|^{\frac{1}{2}}}
 \nonumber \\
 && \times e^{-\frac{1}{2\sigma^2}\Bigl[y^{'n}\bigl\{(I+K)_{n \times n}^{-1}+(I+K)_{n \times n}^{-1}V_{n \times n}(I+K)_{n \times n}^{-1}\bigr\}y^n\Bigr]},
 \label{statehgamma}
 \end{eqnarray} 
 where 
 \begin{eqnarray}
h(\gamma) &=& e^{-\frac{1}{2\sigma^2}\Bigl[-2\gamma y^{'n}\frac{(I+K)_{n \times n}^{-1}V_{n \times n}}{\delta^2}1 + \gamma^21'\Bigl(\frac{I}{\delta^2}-\frac{V_{n \times n}}{\delta^4}\Bigr)1^n\Bigr]}.
\label{margforgamma}
\end{eqnarray} 

\noindent
and Clause II:

 \begin{eqnarray}
 \label{stategdelta}
\mathcal{L}_2(y^n|\gamma, \delta^2, \sigma^2) &=& g(\delta^2) \times \frac{1}{(2\pi\sigma^2)^{\frac{n}{2}}  |I+K|^{\frac{1}{2}}} e^{-\frac{1}{2\sigma^2}[y^{'n}(I+K)_{n \times n}^{-1}y^n]},
\end{eqnarray} 
where
\begin{eqnarray}
&& g(\delta^2) 
= \frac{\Bigl|\Bigl\{(I+K)_{n \times n}^{-1}+(\delta^2I_{n \times n})^{-1}\Bigr\}\Bigr|^{\frac{1}{2}}}{(\delta^2)^{\frac{n}{2}}} \times \nonumber\\
  && e^{\frac{1}{2\sigma^2}\Bigl[y^{'n}(I+K)_{n \times n}^{-1}\bigl\{(I+K)_{n \times n}^{-1}+(\delta^2I_{n \times n})^{-1}\bigr\}^{-1}(I+K)_{n \times n}^{-1}y^n + \frac{2\gamma}{\delta^2} y^{'n}(I+K)_{n \times n}^{-1}\bigl\{(I+K)_{n \times n}^{-1}+(\delta^2I_{n \times n})^{-1}\bigr\}^{-1}1^n}\Bigr]\nonumber\\
 && \times e^{\frac{1}{2\sigma^2}\Bigl[ \frac{\gamma^2}{\delta^4}1^{'n}\bigl\{(I+K)^{-1}+(\delta^2I_{n \times n})^{-1}\bigr\}^{-1}1-\frac{\gamma^2}{\delta^2}1^{'n}1^n\Bigr]}.
 \label{deltasq}
\end{eqnarray}
\end{theorem}

{\bf Remark:} Clause I lets us find the maximum likelihood estimator
$\hat{\gamma}_{MLE}$ by looking only at $h(\gamma)$ while Clause II lets us find
$\hat{\delta}_{MLE}$ by looking only at $g(\delta^2)$.

\begin{proof}
A complete proof is given in the Appendix, Sec. \ref{GPderive}.
\end{proof}

To use Theorem \ref{gammadelta} to find estimates of $\gamma$ and $\delta^2$ we start with $\gamma$,
taking logarithms on both sides of 
\eqref{margforgamma} to get
\begin{equation}
 \label{lnhgamma}
    \ln h(\gamma) = -\frac{1}{2\sigma^2}\Biggl[-2\gamma y^{'n}\frac{(I+K)_{n \times n}^{-1}V_{n \times n}}{\delta^2}1^n+\frac{\gamma^2}{\delta^2}1^{'n}\Biggl(I_{n \times n}-\frac{V_{n \times n}}{\delta^2}\Biggr)1^n\Biggr]. 
\end{equation}
Differentiating \eqref{lnhgamma} with respect to $\gamma$, and equating it to zero gives
\begin{eqnarray*}
\frac{d}{d\gamma}\log_e h(\gamma) = 0 .
\end{eqnarray*}
So we have that
\begin{eqnarray}
&&\frac{2}{2\sigma^2}\frac{y^{'n}(I+K)_{n \times n}^{-1}V1}{\delta^2} - \frac{2\gamma}{2\sigma^2\delta^2}1^{'n}\Bigl(I_{n \times n}-\frac{V_{n \times n}}{\delta^2}\Bigr)1^n = 0
\nonumber\\
&&\implies \gamma1^{'n}\Bigl(I_{n \times n}-\frac{V_{n \times n}}{\delta^2}\Bigr)1^n =  y^{'n}(I+K)_{n \times n}^{-1}V_{n \times n}1^n 
\nonumber\\
&&\implies \hat{\gamma} = \frac{y^{'n}(I+K)_{n \times n}^{-1}V_{n \times n}1^n}{1^{'n}\Bigl(I_{n \times n}-\frac{V_{n \times n}}{\delta^2}\Bigr)1^n} ,
\label{gammahat}
\end{eqnarray}
in which it is seen that $\sigma^2$ does 
not appear.
The second derivative is
\begin{eqnarray*}
  \frac{d^2\ln h(\gamma)}{d\gamma^2} = -\frac{1}{\sigma^2\delta^2}1^{'n}\Biggl(I_{n \times n}-\frac{V_{n \times n}}{\delta^2}\Biggr)1^n
\end{eqnarray*}
which is typically less than 0 because, as we will see, $\delta^2$ is usually small.  Hence, our solution to \eqref{gammahat} will typically be a local maximum.

Next, we use \eqref{deltasq} to help find a good 
estimate of $\delta$.
Unfortunately, we cannot simply differentiate 
$g(\delta^2)$,
set the derivative to zero, and solve.  The 
resulting
equations are just too complicated to be 
useful in any
obvious way.
So, we did a grid search over interval $\mathbf{I} \subset \mathbb{R^+}$ to maximize $g$.  In computational work not shown here, we found that the optimal $\delta \in \mathbf{I}$ was almost always the left hand end point, even 
as $\mathbf{I}$ moved closer and closer to $0$.
In the limit of $\delta \rightarrow 0$, 
$\hat{\gamma} \rightarrow  0$ as well.  
This suggests that the mean
and variance of the bias $a$ are
zero i.e., there is no bias.

Even though $\sigma^2$ appears in \eqref{deltasq},
we always pragmatically
set $\delta$ to be small so that in
our computations here
the bias would not overwhelm the data.  
For instance, we typically set $\delta = .1$
and tested larger values like $\delta=1$.
When we recomputed with larger values we
typically found that the predictive error
increased very slowly.

\subsection{Dirichlet Process prior prediction}
\label{DPPmethod}

Suppose a discrete prior $G$ is distributed according to
a Dirichlet Process and write
$G \sim DP(\alpha, G_0)$ 
where $\alpha$ is the mass parameter 
and $G_0$ is the 
base measure with $\mathbb{E}(F) = G_0$.
Then, by construction, we have the following; 
see \cite{sghoshal}.
If the sample space $\mathbb{R}$ 
is partitioned into $A_1, A_2, \cdots, A_k$, 
then the random vector of probabilities 
$(G(A_1), G(A_2), \cdots, G(A_k))$ follows a 
Dirichlet distribution,
i.e., $p(G(A_1), G(A_2), \cdots, G(A_k)) \sim Dirichlet(\alpha(A_1), \alpha(A_2), \cdots, \alpha(A_k))$,
where $\alpha(\mathbb{R}) = M$, which we take here to be one. 

Now, 
the posterior distribution of $G(A_1), G(A_2)$, $\cdots, G(A_k)|Y_1, Y_2, \cdots, Y_n$ is also Dirichlet but with parameters $\alpha(A_j)+n_j$ where, $n_j = \sum_{i=1}^{n}I(Y_i \in A_j); j = 1,2, \ldots, k$. 
If $Y_j^{'}; j = 1,2, \cdots, k$ are the distinct observations 
in $\{Y_i; i = 1,2, \cdots, n\}$, the posterior predictive 
distribution of $Y_{n+1}|Y_1, Y_2, \cdots, Y_n$ is
 \begin{eqnarray}
   Y_{n+1}|Y_1, Y_2, \cdots, Y_n = 
\begin{cases}
    \delta_{Y_j^{'}}, \text{with probability } {\frac{n_j}{M+n}; j = 1,2, \cdots, \text{k; and} } \\
    G_0,               \text{with probability } {\frac{M}{M+n}} 
\end{cases} .
\nonumber
\end{eqnarray} 
Thus, our Dirichlet Process Prior (DPP) 
predictor is 
\begin{eqnarray}
\label{DPPPredictor}
    \hat{Y}_{n+1} &=& \sum_{j=1}^{k} y'_j \frac{n_j}{M+n} + \frac{M}{M+n} median(G_0) .
\end{eqnarray} 

\section{Shtarkov Solution Based Predictors}
\label{Shtarkov}

We distinguish between the Shtarkov \textit{solution}
that solves a
specific optimization problem giving the normalized
maximum likelihood estimator as the minimax
regret solution and the Shtarkov \textit{predictor}
that is the ratio of two Shtarkov solutions.  The
latter can be derived explicitly for the normal case
when the variance is known
and we use it as one of our predictors.  

Here, for the sake of completeness, we give the Shtarkov 
solution and predictors in general.  Then we look at
special cases to present the predictor we actually use
in our computational comparisons.

 \subsection{The Shtarkov solution}
 \label{Shtsolution}
 
Consider a game being played between Nature $N$ and a Player $P$. $P$ has access to experts indexed by $\theta \in \Theta \subset \mathbb{R}^k$. The goal of $P$ is to make 
the best prediction of the value that $N$ issues. 
Let us consider the univariate case. Suppose $P$ can call on
experts and they provide their best predictive distributions
$p(\cdot|\theta)$.  After receiving these, $P$ announces 
the prediction
$q(\cdot)$.  In practice, $P$ might choose $q(\cdot)$ 
to match the performance of the best expert $\theta$.

Assume the $y_i$'s are from a univariate 
data stream $y_1, y_2, \ldots$
issued by $N$.  $N$ can issue $y_i$'s by any rule s/he wants
or, here, by no rule, probabilistic or otherwise, since
we are regarding the $y_i$'s as ${\cal{M}}$-open.
Regardless of how $N$ generates data, after  the
$n^{th}$ step, $P's$ cumulative regret with respect to expert $\theta$ is given by  
\begin{equation}
    \log\frac{1}{q(y^n)}-\log\frac{1}{p(y^n|\theta)} = \log\frac{p(y^n|\theta)}{q(y^n)} ;\label{regret}
\end{equation} 
If $P$ wants to minimize the maximum regret,
s/he chooses
\begin{equation}
    q_{\sf opt}(y^n) = \arg\underset{q}{\min}\, \underset{y^n}{\sup}\, \underset{\theta}{\sup }\, \log\frac{p(y^n|\theta)}{q(y^n
    )} = \frac{p(y^n|\hat{\theta})}{\int p(y^n|\hat{\theta})dy^n}\label{maxreg} 
\end{equation}
where $\hat{\theta} = \hat{\theta}(y^n)$ where $\hat{\theta}$ is the maximum likelihood estimator, provided the integral exists; see \cite{Shtarkov:1988} and \cite{Rissanen:1996}. 
The normalized maximum likelihood $q_{\sf opt}$ is called
the (frequentist) Shtarkov solution.
If a weighting function $w(\theta)$
across experts is given then the
Bayesian form of \eqref{maxreg} is
\begin{equation}
    q_{\sf opt, B}(y^n) = \frac{w(\tilde{\theta}(y^n))p(y^n|\tilde{\theta}(y^n))}{\int w(\tilde{\theta}(y^n))p(y^n|\tilde{\theta}(y^n))dy^{n}}
    \label{ShtarkovslnB}
\end{equation}
where $\tilde{\theta}$ is the posterior mode. 

\subsection{The Shtarkov Predictors}
\label{Shtpredictor}

We take as our frequentist
Shtarkov point predictor the mode of 
\begin{eqnarray}
    q_{\sf Sht}(y_{n+1}) = \frac{q_{opt}(y^{n+1})}{q_{opt}(y^n)},
    \label{approximateshatsoln}
\end{eqnarray}
the ratio of two Shtarkov solutions.  
The analogous ratio denoted 
$q_{\sf Sht, B}(y_{n+1})$ using \eqref{ShtarkovslnB}
gives the Bayes Shtarkov point predictor.
Expression \eqref{approximateshatsoln} looks like a conditional
density but in fact is just a distribution because Shtarkov
solutions don't marginalize properly.  Here,
we use the mode
of the numerator given that $y^n$ is fixed.
The mode turns out to be a good predictor -- better than the
mean or median because often $q_{Sht}$ is often
highly skewed (see \cite{Le:Clarke:2017}). 
In a heuristic sense,
$q_{Sht}$ can be regarded
as an approximation to a conditional 
density for $y_{n+1}$, i.e.,
$q(y_{n+1} \vert y^n)$ if it were to exist.

\subsection{Special Cases}
\label{shtexp}
Different examples of Shtarkov solutions and predictors arise from choosing different classes of experts -- and weights 
in the Bayesian case.  Here we limit ourselves to 
exponential families.  We record the following 
proposition without proof
because it follows from applying the definitions.

\begin{t1}
\label{propshtarkov}
Let $p$ be a full-rank exponential family in the 
canonical parametrization $\eta$, i.e., 
\begin{eqnarray}
\label{expfam}
    p(y|\eta) &=& h(y)exp(\eta^{T}T(y)-A(\eta)),
\end{eqnarray}
where $T$ is the natural sufficient statistic and $A$ is the normalizing constant.  Then:
\begin{enumerate}
    \item If $\mu(\eta) = {\mathbb E}_{\eta}(T(Y))$, we find that $q_{\sf opt}$ is
    \begin{eqnarray}\label{qoptexp}
        q_{opt}(y^n) &=& \frac{p({y}^n|\mu^{-1}(\frac{1}{n}\sum_{i=1}^{n}T(Y_i)))}{\int_{{y}^n}^{} p({y}^n|\mu^{-1}(\frac{1}{n}\sum_{i=1}^{n}T(Y_i)))d{y}^n} . \nonumber\\
    \end{eqnarray}
    \item Let $p(\eta)$ be a conjugate prior 
    for \eqref{expfam} with hyperparameters denoted $\beta$ and
    $\nu$. In the Bayes case we have
    \begin{eqnarray}
        q_{\sf opt, B}({y}^n) &=& \frac{p(\hat{\eta}|\beta, \nu) \times p({y}^n|\hat{\eta})}{\int_{{y}^n}^{}p(\hat{\eta}|\beta, \nu) \times p({y}^n|\hat{\eta})d{y}^n}
    \end{eqnarray}
    where $\hat{\eta}$ is $\mu^{-1}$ of the posterior mode. 
\end{enumerate}
\end{t1}
\begin{proof}
Omitted.
\end{proof}

Instances of the Shtarkov predictor can be
worked out for several cases using Prop. \ref{propshtarkov}.  Here, we only use the 
frequentist normal predictor and only consider
two cases:  i) $\mu$ unknown and 
$\sigma$ known, and ii) both $\mu$ and
$\sigma$ unknown.   As we shall see, the 
Shtarkov point predictor
for these two cases is the same.

We start with case i).  
For data $y_1, y_2, \ldots, $ write 
$\bar{y} = \bar{y}_n$
for the sample mean from the first $n$ 
observations.  The 
frequentist Shtarkov solutions 
for $y^n$
is the 
normalized version of the maximum likelihood
which for $n+1$ is
\begin{eqnarray}
\label{likeynplus1}
   p(y^{n+1}|\hat{\mu}_{n+1},\sigma^2) &=& \biggl(\frac{1}{\sigma^2 2\pi}\biggr)^{\frac{n+1}{2}} e^{-\frac{1}{2\sigma^2}\sum_{i=1}^{n+1}(y_i-\bar{y}_{n+1})^2}  
\end{eqnarray}
where $\hat{\mu} = \bar{y}_n$ is the MLE.
So, if we write
\begin{eqnarray}
\label{recursionyn1bar}
    \bar{y}_{n+1} = \frac{n\bar{y}_n+y_{n+1}}{n+1}.
\end{eqnarray}
and use \eqref{recursionyn1bar} in \eqref{likeynplus1}, we get
\begin{eqnarray}
    \ln p(y^{n+1}|\hat{\mu}_{n+1},\sigma^2) &=&
 -\frac{n+1}{2}\ln(\sigma^2 2\pi)-\frac{1}{2\sigma^2}\sum_{i=1}^{n} y_i^2 + \frac{2n\bar{y}_n}{2\sigma^2}\frac{n\bar{y}_n+y_{n+1}}{n+1} \nonumber\\
 && -\frac{n}{2\sigma^2}\frac{(n\bar{y}_n+y_{n+1})^2}{(n+1)^2}-\frac{y^2_{n+1}}{2\sigma^2} + \frac{2y_{n+1}}{2\sigma^2}\frac{n\bar{y}_n+y_{n+1}}{n+1} \nonumber\\
 && -\frac{1}{2\sigma^2}\frac{(n\bar{y}_n+y_{n+1})^2}{(n+1)^2}.
\end{eqnarray}

Taking the derivative, and setting it equal to zero, and re-arranging, gives that solving
$\frac{d}{dy_{n+1}}\ln p(y^{n+1}|\hat{\mu}_{n+1}|\sigma^2) = 0$ 
leads to
\begin{eqnarray*}
    y_{n+1}\biggl[\frac{2}{\sigma^2(n+1)}-\frac{1}{\sigma^2}-\frac{n+1}{\sigma^2(n+1)^2}\biggr] = \frac{n^2 \bar{y}_n}{\sigma^2(n+1)^2} + \frac{n\bar{y}_n}{\sigma^2(n+1)^2} -\frac{2}{\sigma^2}\frac{n}{n+1}\bar{y}_n.
    \end{eqnarray*}
    Now, we find
    \begin{eqnarray}
    \hat{y}_{n+1} &=& \frac{ \frac{n^2 \bar{y}_n}{\sigma^2(n+1)^2} + \frac{n\bar{y}_n}{\sigma^2(n+1)^2}-\frac{2}{\sigma^2}\frac{n}{n+1}\bar{y}_n}{\frac{2}{\sigma^2(n+1)}-\frac{1}{\sigma^2}-\frac{n+1}{\sigma^2(n+1)^2}}. \nonumber\\
    \implies  \hat{y}_{n+1} &=& \frac{ \frac{n(n+1) \bar{y}_n}{(n+1)^2} -\frac{2n}{n+1}\bar{y}_n}{\frac{2}{(n+1)}-1-\frac{n+1}{(n+1)^2}}
    = - \frac{\frac{n \bar{y}_n}{n+1}}{\frac{1}{n+1}-1} = \bar{y}_n.
\end{eqnarray}
Hence, the frequentist Shtarkov pointwise predictor
with normal experts is simply the mean, independent
of the value of $\sigma$.

In case ii), where both $\mu$ and $\sigma^2$ 
are unknown,  we have  $\hat{\mu}_n = \bar{y}_n$ and, $\hat{\sigma}^2_n = \frac{1}{n}\sum_{i=1}^{n}(y_i-\mu)^2$.
Then, 
\begin{eqnarray}
    p(y^n|\hat{\mu}_n, \hat{\sigma^2}_n) 
    &=& \frac{n^{\frac{n}{2}}e^{-\frac{n}{2}}}{(2\pi)^{\frac{n}{2}}\sum_{i=1}^{n}(y_i-\bar{y}_n)^2 } \nonumber\\
    & \propto & \frac{1}{\sum_{i=1}^{n}(y_i-\bar{y}_n)^2}.
\end{eqnarray}
Hence
\begin{eqnarray}
    p(y^{n+1}|\hat{\mu}_{n+1},\hat{\sigma^2}_{n+1}) \propto \frac{1}{\sum_{i=1}^{n+1}(y_i-\bar{y}_{n+1})^2}.
    \label{musigmasqnplus1}
\end{eqnarray}
Taking logarithms on both sides of \eqref{musigmasqnplus1}, we have
\begin{eqnarray}
     \ln p(y^{n+1}|\hat{\mu}_{n+1},\hat{\sigma^2}_{n+1}) &\propto& -\ln\sum_{i=1}^{n+1}(y_i-\bar{y}_{n+1})^2.
     \label{logmusigmaynplus1}
\end{eqnarray}
Using \eqref{recursionyn1bar} in \eqref{logmusigmaynplus1}, we get
\begin{eqnarray}
    \ln p(y^{n+1}|\mu,\hat{\sigma^2}_{n+1}) &\propto& -\ln\Biggl[\sum_{i=1}^{n}\biggl(y_i-\frac{n\bar{y}_n+y_{n+1}}{n+1}\biggr)^2 \nonumber\\
    && + \biggl(y_{n+1}-\frac{n\bar{y}_n+y_{n+1}} {n+1}\biggr)^2\Biggr]\nonumber\\
    &=& -\ln\Biggl[\sum_{i=1}^{n}\biggl(y_i-\frac{n\bar{y}_n+y_{n+1}}{n+1}\biggr)^2 \nonumber\\
    && + \frac{1}{(n+1)^2} \biggl(ny_{n+1}-n\bar{y}_n\biggr)^2\Biggr].
\end{eqnarray}

Again, differentiating and setting the
derivative equal to zero, i.e., solving
$ \frac{d}{dy_{n+1}}\ln p(y^{n+1}|\hat{\mu}_{n+1}, \hat{\sigma^2}_{n+1}) = 0$, gives 
 \begin{eqnarray}
    && \sum_{i=1}^{n}2\biggl(y_i-\frac{n\bar{y}_n+y_{n+1}}{n+1}\biggr)\biggl(-\frac{1}{n+1}\biggr)+ \frac{2n}{(n+1)^2} \biggl(ny_{n+1}-n\bar{y}_n\biggr) = 0 \nonumber\\
  &&   \implies \hat{y}_{n+1} = \bar{y}_n.
  \end{eqnarray}

Hence, in the normal mean cases we get the same
point predictor, the sample mean of the past data.
We use this in our computations in Sec. \ref{compresults}.  Shtarkov point predictors,
Bayes and frequentist, can be found in many other
exponential families
cases, but not in general in closed form.

\section{Computational comparisons}
\label{compresults}

To present our computational results we
begin by listing our predictors.  Then we 
describe the settings for our comparisons.
Finally, we present our computations and
interpret what they imply about the methods.

\subsection{Our predictors}
\label{methods}

We computationally compare the predictors 
that have
been presented in the earlier sections.
There were two predictors based on hash 
functions.
These HPB's used the mean and the median
of the empirical DF generated by the 
{\sf Count-Min} sketch.  They were
explicitly given by
\eqref{CMpredictionmean} and \eqref{CMpredictionmed} in Sec. \ref{HBPmethod}.
There were three Bayesian predictors
namely GPP's with no bias, GPP's with
a random additive bias, and DPP's.
They were given in 
\eqref{mustarNRB}, Theorem \ref{studentt}, 
and \eqref{DPPPredictor}.
The predictor in \eqref{mustarNRB} requires
the estimation of parameters as discussed
in Subsec. \ref{GPPRV}.
The predictor in Theorem \ref{studentt}
was denoted $A_1$ and was not given explicitly
but $g_1^n$ and $\gamma_1$ can be found
from \eqref{ygamma1} while $\gamma_2$ 
is can be found from \eqref{ygamma2}.
The estimation of parameters required 
to use GPP's with a random additive bias 
is discussed in Subsec. \ref{GPPRB}.
The `parameter' $G_0$ in DPP predictors
has to be chosen and is user dependent.
Finally, we used one frequentist Shtarkov 
point predictor based on
normality.  It was simply the mean,
as derived in Subsec. \ref{shtexp}.

\subsection{Settings for the comparisons}
\label{specifics}

We compare point predictors by their 
cumulative $L^1$ error.  That is, for each 
method,  we have a sequence of errors
$\vert y_i - \hat{y}_i \vert$ where
$\hat{y}_i$ depends on $y_1, \ldots , y_{i-1}$
(and possibly a burn-in set ${\cal{D}}_b$)
and we find the cumulative predictive error
\begin {eqnarray}
CPE = CPE(n) = \frac{1}{n} \sum_{i=1}^n \vert y_i - \hat{y}_i \vert .
\label{CPE}
\end{eqnarray}
It is seen that
\begin{eqnarray}
CPE(n+1) =  \frac{1}{n+1} \left( n CPE(n) + \vert y_{n+1} - \hat{y}_{n+1} \vert \right)
\nonumber
\end{eqnarray}
so it is easy to update the CPE from time step to
time step.
For each method, we report the final $CPE$.

Since we are using HBP's, it is natural to
exploit the fact they can be computed in
one pass.  We can do this readily for the
Shtarkov predictor and the DPP predictor.
However, it is difficult to do this for either
of the GPP predictors because the variance
matrix increases in dimension.  

So, to include GPP's in our comparisons
we have to ensure the variance matrices in the 
GPP predictors do not increase excessively.
We do this by using what we call a representative
subset of fixed size that is updated from time
step to time step.  Essentially, we use
the cluster centers from streaming $K$-means
for a fixed choice of $K$, here $K=200$.
Under streaming $K$-means, the cluster centers 
at time step $n$ update easily to give
the cluster centers at time step $n+1$.
We then use the cluster centers at time $n$
as the data to form our predictors for time $n+1$.

Thus, we have two sets of comparisons of
$CPE$'s, one for the four methods that can 
be implemented
in one pass and another for all six methods
using streaming representative subsets.
In fact, we compare all of them in an effort to understand how the various methods behave.

We use different forms of the three predictors
for different data sets.  However, the quantities
that must be chosen are the same in all cases.
For the HPB methods (mean and median), 
we must choose
$K$, $d_K$, and $W_K$.  For the Bayesian
methods our choices are as follows.
For GPP, we chose the variance matrix 
$K_{n\times n}$ to be of the form of a 
correlation matrix for an $AR(1)$ time series.
That is, for given correlation $\rho$ we used
\begin{eqnarray}
K_{n \times n} = 
\left(\begin{array}{cccc}
1 & \rho & \cdots & \rho^{n-1}\\
\rho & 1 & \cdots & \rho^{n-2}\\
\vdots & \vdots & \ddots & \vdots\\
\rho^{n-1} & \rho^{n-2} & \cdots & 1
\end{array}\right)
\nonumber
\end{eqnarray}
and set $\rho = .8$.  
For the GPP with random bias, we used
$A_1$ as our point predictor and so
only had to find
values for $g_1^n$, $\gamma_1$, $\gamma_2$,
$\delta$ and $\sigma$.  We listed our choices 
at the end of Subsec. \ref{GPPRB}.
For DPP's we chose the base measure
$G_0$ to be a Discrete Uniform Distribution 
on the range $[\min\{y_1, \ldots, y_n\},
\max\{y_1, \ldots , y_n\}]$.  For the Shtarkov
predictor in the normal case, we got an 
expression involving data only. The predictor was fully specified once the family of experts 
and weighting had been fixed.

To finish the general specifications of our
predictors, we initialized all our sequences
of predictions using 10\% of the total data we
intended to use.  Thus, in the {\sf Columbia}
rainfall data below where we had $n=5000$, we
used a burn-in of the first 500 data points
to form the predictions for each of
the 501 time step.
These then gave us the first terms in our $CPE$'s
for the ten methods.

\subsection{Results}
\label{results}

In this subsection we use the ten methods 
described in Subsec. \ref{methods} on
four real data sets.  The first three are rainfall
data sets from three jurisdictions here
called {\sf Columbia}, {\sf India}, and
{\sf Bhubaneshwar}.  Note that because our
methods are designed for ${\cal{M}}$-open
problems, simulated data will 
not be complex enough in general.  Indeed,
in other computational work not shown here,
we found a very different ordering of
techniques by performance.  The techniques
designed for streaming data performed
relatively poorly.

Another comment:  There are many classes of
complex data and it is at this time nearly
impossible to characterize when each 
(or indeed any) method works better than the
others.  So, we have chosen a collection
of data sets that, as will be seen,
illustrate our general points.

The {\sf Columbia} data set can be
found at \url{https://data.world/hdx/f402d5ef-4a74-4036-8829-f04d6f38c8e9}.  The dataset contains daily values of total precipitation (mm) in 
Columbia over a period of four years 
ending in the year 2013.  They were collected 
from 27 different base stations and the `time' 
index cycles through them.  We suggest this 
cycling will be typical of many kinds of streaming
data. We use the first 
5000 rows of the {\sf value} column of the 
dataset.  This data set, like {\sf India}
below, is not
a pure time series -- it's as if there were a 
mixing distribution over the stations.
However, there is a pattern that would allow
prediction so this is a fair test of how
well a predictor can perform on complex data.

The {\sf India} data set is similar and
can be found at \url{https://data.world/hdx/687c4f99-6ec6-4b30-ada2-a5a0f9eac629}. Parallel to {\sf Columbia}, 
this dataset contains values of daily total 
precipitation (mm) cycling over 76 different 
base stations. Measurements of total 
precipitation for a two year period (2010-2011)
can be found in the dataset.  Again we use 
the first 5000 rows of the {\sf value} 
column of the dataset
For the HBP's computations with
$\sf Columbia$ and $\sf India$, $d_K =10$,
$W_K = 50$, and $K = 100$.
We set $K= \lceil n/50 \rceil$ in all cases
but simply chose
$d_K$ and $W_K$ larger for larger sample sizes.

The {\sf Bhubhneshwar} data set can be found at \url{https://www.kaggle.com/datasets/vanvalkenberg/historicalweatherdataforindiancities}. It
has daily precipitation data (mm) 
from 01/01/1990 to 07/20/2022. The column 
{\sf prcp} was used for getting the values 
of CPE. Rows with missing values were 
deleted leaving 
6838 data points.  
For the computations with $\sf Bhubhneshwar$, $d_K = 15$, $W_K = 50$, $K = 137$.
This data set like 
{\sf accelerometer} below, is a time series.

The fourth data set is drawn from the phones
{\sf accelerometer} benchmark data that
can be found at \cite{datasource}, 
which provides a complete 
description.  
We extracted the first 10,000
rows of the data set and used the column $``y"$ for our computing. We split the data into 
four parts, i.e., sets of 2500 each,
and computed the results. 
For the HBP computations with 
{\sf accelerometer}, $d_K = 10$, $W_K = 20$, and $K = 50$.

In our tables, we follow the convention that
the numbers in \textbf{bold} denote the smallest errors and the asterisk (*) represents the second best. Headings indicate whether the error
in a column is from a one-pass method or used
a representative subset from $K$-means.  We abbreviate the names of our methods as Sht, DPP, GPP(RB) and GPP(no RB) 
to mean the Shtarkov (Normal), Dirichlet process prior, and GPP with and without random bias, respectively.

Turning to the numerical results, we begin with
the $CPE$'s for {\sf Columbia} given in Table \ref{columbia}.  In this case we see exactly
the pattern of errors that we expect.
Namely, the one-pass HBP 
median has the lowest error
and GPP(RB) has the second lowest error.
The other methods performed notably worse.
We attribute the good performance of GPP(RB)
to the extra spread from the random additive bias
and the poor performance of Shtarkov to
its extreme simplicity.

\begin{table}[ht!]
\begin{tabular}{@{}llllllllll@{}}
\multicolumn{3}{c}{One pass} &
      \multicolumn{7}{c}{Representative}   
\\\cmidrule(lr){1-4}\cmidrule(lr){5-10}
\begin{tabular}[c]{@{}l@{}}HBP\\ (Mean)\end{tabular} & \begin{tabular}[c]{@{}l@{}}HBP\\ (Median)\end{tabular} & \begin{tabular}[c]{@{}l@{}}Sht\\ \end{tabular} & DPP      & \begin{tabular}[c]{@{}l@{}}HBP\\ (Mean)\end{tabular} & \begin{tabular}[c]{@{}l@{}}HBP\\ (Median)\end{tabular} & \begin{tabular}[c]{@{}l@{}}Sht\\ \end{tabular} & DPP      & \begin{tabular}[c]{@{}l@{}}GPP\\ (RB)\end{tabular} & \begin{tabular}[c]{@{}l@{}}GPP\\ (no RB)\end{tabular} \\ \midrule
1006.8                                             & \textbf{944.8}                                               & 986.8                                                  & 989.1 & 1049.7                                           & 960.1                                               & 959.7                                                 & 985.8 & 947.2*                                           & 1000.0                                              \\ \bottomrule
\end{tabular}
\caption{\label{columbia} Final $CPE$'s for 
the ten predictors using the {\sf Columbia} rainfall data. }
\end{table}

Table \ref{India} presents the final $CPE$'s
for the {\sf India} data.  It is seen that
the best methods have $CPE$ around 1050.
These are one pass HBP median, one-pass Shtarkov,
and GPP (RB).  The only possible
surprise here is that
one-pass Shtarkov is doing so well.  However,
this does not contradict our basic
inference that the top methods for 
the class of data this example represents
are one-pass HBP median and GPP(RB).

\begin{table}[ht!]
\begin{tabular}{@{}llllllllll@{}}
\multicolumn{3}{c}{One pass} &
      \multicolumn{7}{c}{Representative}   
\\\cmidrule(lr){1-4}\cmidrule(lr){5-10}
\begin{tabular}[c]{@{}l@{}}HBP\\ (Mean)\end{tabular} & \begin{tabular}[c]{@{}l@{}}HBP\\ (Med)\end{tabular} & \begin{tabular}[c]{@{}l@{}}Sht\\ \end{tabular} & DPP      & \begin{tabular}[c]{@{}l@{}}HBP\\ (Mean)\end{tabular} & \begin{tabular}[c]{@{}l@{}}HBP\\ (Median)\end{tabular} & \begin{tabular}[c]{@{}l@{}}Sht\\ \end{tabular} & DPP      & \begin{tabular}[c]{@{}l@{}}GPP\\ (RB)\end{tabular} & \begin{tabular}[c]{@{}l@{}}GPP\\ (no RB)\end{tabular} \\ \midrule
1231                                             & 1052                                            & \bf{1049}                                              & 1151 & 1120                                              & 1171                                               & 1227                                              & 1237 & 1050*                                           & 1066                                            \\ \bottomrule
\end{tabular}
\caption{\label{India} Final $CPE$'s for 
the ten predictors using the {\sf India} rainfall 
data. }
\end{table}

Table \ref{Bhuv} presents the final $CPE$'s
for the {\sf Bhubhneshwar} data.  It is seen that
again the one-pass HBP median predictor
is best and ties with the representative
HBP median predictor.  Second place goes to GPP
with no RB rather than GPP (RB), although
GPP (RB) comes in third.  

\begin{table}[ht!]
\begin{tabular}{@{}llllllllll@{}}
\multicolumn{3}{c}{One pass} &
      \multicolumn{7}{c}{Representative}   
\\\cmidrule(lr){1-4}\cmidrule(lr){5-10}
\begin{tabular}[c]{@{}l@{}}HBP\\ (Mean)\end{tabular} & \begin{tabular}[c]{@{}l@{}}HBP\\ (Med)\end{tabular} & \begin{tabular}[c]{@{}l@{}}Sht\\ \end{tabular} & DPP      & \begin{tabular}[c]{@{}l@{}}HBP\\ (Mean)\end{tabular} & \begin{tabular}[c]{@{}l@{}}HBP\\ (Median)\end{tabular} & \begin{tabular}[c]{@{}l@{}}Sht\\ \end{tabular} & DPP      & \begin{tabular}[c]{@{}l@{}}GPP\\ (RB)\end{tabular} & \begin{tabular}[c]{@{}l@{}}GPP\\ (no RB)\end{tabular} \\ \midrule
 9.619                                              & \textbf{7.102}                                               & 9.633                                                 & 9.934 & 10.09                                               & \textbf{7.102}                                                  & 10.077                                                & 9.609 & 8.613                                              & 7.460*                                                 \\ \bottomrule
\end{tabular}
\caption{\label{Bhuv} Final $CPE$'s for 
the ten predictors using the {\sf Bhubhneshwar}
rainfall data. }
\end{table}

To explain this, we suggest that the 
{\sf Bhubhneshwar} data
is slightly less complex than the
{\sf Columbia} data and hence a little bit 
easier to predict.  
So, we plotted the 
{\sf Columbia} and {\sf Bhubhneshwar} data
as time series.  This is shown in Fig. \ref{fig1}.
Although hard to see at the
scale of this plot, the {\sf Bhubhneshwar} data
shows more regularity than the {\sf Columbia}
data which looks much more patternless.   
Since patterns can indicate structure to
improve prediction, the prediction problem
of {\sf Bhubhneshwar} may be a little easier
so that the random bias term does not help
the GPP.  The pattern in the time series plot
may also ensure that a representative subset
from $K$-means really is representative enough
to help prediction substantially.

\begin{figure}[ht]
  \centering

\includegraphics[width=135mm,scale=.7]
{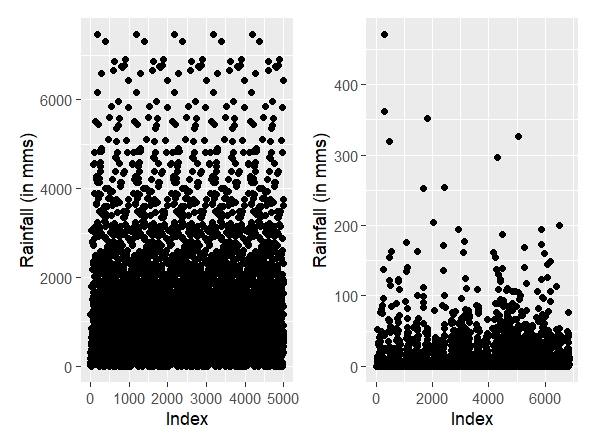}
\caption{\label{fig1} Left: Plot of the
{\sf Columbia} data as a time series.
Right: Plot of the {\sf Bhubhneshwar} data
as a time series.  } 
\end{figure}

Table \ref{realtable} presents the final $CPE$'s
for the {\sf accelerometer} data.  The third 
and fourth rows are in accord with our 
expectations, except that the two methods,
one-pass median HBP and GPP (RB) tie
and tie with other methods including the
Shtarkov predictor with a representative subset.

Rows one and two bear some comment.
First, row two shows that the one-pass HBP median
performs best but the GPP (RB) is third,
being outperformed by the DPP which is second.
Second, row one shows GPP (RB) is best
while one-pass HBP median does worse than some
methods (although better than others).

\begin{table}[ht!]


\begin{tabular}{@{}llllllllll@{}}
\multicolumn{3}{c}{One pass} &
      \multicolumn{7}{c}{Representative}   
\\\cmidrule(lr){1-4}\cmidrule(lr){5-10}
\begin{tabular}[c]{@{}l@{}}HBP\\ (Mean)\end{tabular} & \begin{tabular}[c]{@{}l@{}}HBP\\ (Median)\end{tabular} & Sht   & DPP    & \begin{tabular}[c]{@{}l@{}}HBP\\ (Mean)\end{tabular} & \begin{tabular}[c]{@{}l@{}}HBP\\ (Median)\end{tabular} & Sht   & DPP    & \begin{tabular}[c]{@{}l@{}}GPP\\ (RB)\end{tabular} & \begin{tabular}[c]{@{}l@{}}GPP\\ (no RB)\end{tabular} \\ \midrule
0.326*                                               & 0.337                                                  & 0.770 & 0.335  & 0.339                                                & 0.347                                                  & 0.328 & 0.335  & \bf{0.308}                                              & 0.339                                                 \\
0.045                                                & \bf{0.032}                                                  & 0.204 & 0.042  & 0.155                                                & 0.070                                                  & 0.076 & 0.040* & 0.074                                              & 0.353                                                 \\
0.069                                                & \bf{0.026}                                                  & 0.094 & 0.027* & 0.065                                                & \bf{0.026}                                                  & 0.029 & \bf{0.026}  & \bf{0.026}                                              & 0.395                                                 \\
0.084                                                & \bf{0.026}                                                  & 0.099 & 0.027* & 0.031                                                & \bf{0.026}                                                  & \bf{0.026} & \bf{0.026}  & 0.027*                                             & 0.383                                                 \\ \bottomrule
\end{tabular}

\caption{\label{realtable} Final $CPE$'s for 
the ten cases using the {\sf accelerometer} data.  Each row of numbers corresponds to a quarter of the first 10,000 data points, in order.}
\end{table}

To explore rows one and two further we plotted
histograms of the two quarters of
data in Fig. \ref{fig2}.
The histogram of the first quarter is trimodal.
It is possible that this makes the data set
broader and hence better captured by GPP (RB)
because the random bias matches the spread of
the data more easily.  The other methods, 
being more purely locations may just not be able to
capture the modes at all -- indeed it is seen that
the range of errors is narrow namely
[.326, .347]
except for one-pass Shtarkov that has error .770
and is completely insensitive to spread.

The histogram of the second quarter is unimodal
with roughly symmetric tails that do not look
heavy.  In fact, it is close to normal apart from
a little bit of left skew.  The one-pass HBP median
may simply converge relatively quickly to a
distribution that matches the histogram while
the other converge more slowly or have their
performances harmed by the skew.

\begin{figure*}[ht!]
  \centering
\includegraphics[width=100mm,scale=0.7]{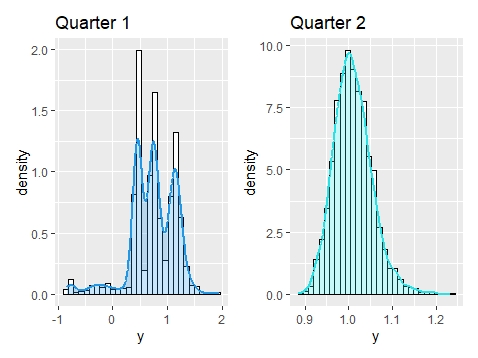}
\caption{\label{fig2}Plot of first two quarters of
the {\sf accelerometer} data.} 
\end{figure*}

\section{Discussion}
\label{discussion}

The main contribution of this paper is
to present and evaluate several predictive 
techniques for complex data, specifically
${\cal{M}}$-open data.  We presented two
new techniques a hash function based predictor
using the {\sf Count-Min} sketch
and a Gaussian process prior with random 
additive bias predictor.  We gave
some of the key properties of these
predictors.  In addition, we derived a 
Shtarkov based
predictor with normal experts.

We compared these predictors computationally.
In our work
we distinguished two cases -- methods that
were one pass and methods that relied on
a representative data set, here obtained as 
the cluster centers from streaming $K$-means.
We argued that for the most complex data,
the one-pass hash function based median
predictor typically performed best in
an $L^1$ sense -- at least for the data
sets we used here.  The Gaussian process
prior with random additive bias predictor
sometimes tied with the one-pass hash-based
median predictor but more often performed
slightly worse; perhaps this occcurs because GPP's represent more of a model than HBP's do. The Shtarkov predictor
sometimes performed well but was likely too
simple a version of the more general class
of Shtarkov predictors to perform well
reliably.  For the data we used, we regard
it as coming third place.
The other methods were not generally intended
for streaming ${\cal{M}}$-open data
and typically did not perform as well
as these three methods.

An important caveat to all our conclusions is
that the class of complex or ${\cal{M}}$-open
data is huge and there are doubtless many 
subclasses that can be defined.  The essential
and only common property the ${\cal{M}}$-open 
data class has is that there is no model
that can be usefully formulated for it.
Consequently, we regard a lot of streaming data
as ${\cal{M}}$-open.  The two new methods we
have proposed here seem best designed for
this sort of data but we are as yet unable to
be more precise as to which subclass of
${\cal{M}}$-open data our methods are
most appropriate for.  Indeed, our
examples only show that there is a large
subclass of data sets for which our methods
are possibly best.  In other work not discussed
here we have found data sets where our methods
perform poorly.  Usually, these are obviously
not ${\cal{M}}$-open e.g., ${\cal{M}}$-closed
meaning they do follow a model, however complex,
that we can identify.  Hence we have not 
shown any simulation results.

As a final point, we ask if there are
data sets that are so complex 
that stochastic assumptions become too strong
for them to satisfy.
After all, assuming a data generator follows
a probability model means that there are
large collections of strings of data that 
get probability zero.
Consider a sequence of real valued responses
$x_1, \ldots, x_n, \ldots$.  If they are
from a $N(0,1)$ then we cannot have
$\bar{x} \rightarrow c$ for any $c \neq 0$.
That is, countable sequences that 
are the limit points of sets having
positive probability in the $n$-fold sample
spaces must correspond to 
a hyperplane in $\mathbb{R}^{\aleph_0}$.
However, the vast majority of $\mathbb{R}^{\aleph_0}$ gets probability zero.
This conflicts with increasing complexity
because we expect that
as a response becomes more complex, 
it should assume a greater number of values.
Indeed, countable sequences 
with a limit are themselves only a fraction 
of the points in $\mathbb{R}^{\aleph_0}$
and it is conceivable that some data
sources may routinely give outcomes that
simply do not have a limit in any commonly 
used sense.

\section{Appendix}
\label{GPderive}

In this Appendix we present the proofs from 
Subsubsec. \ref{streamingGC} and Subsec.
\ref{GPPRB}.  We begin by proving
Theorem \ref{GCtheorem}.  Then we turn to
proofs for 
Theorem \ref{studentt} and 
Theorem \ref{gammadelta}.

\subsection{Proof of Theorem \ref{GCtheorem}}

\begin{proof}
Throughout the proof we treat $n$, $K$, $d_K$, 
and $W_K$
as fixed but large, only letting them 
increase compatibly at the end of each result. 
We begin with the proof of Clause I; it requires the following.

\begin{l1}
\label{lemma1gl}
Let $F$ be an arbitrary distribution function 
on $[m,M]$ and let $\epsilon >0$. Then, there is
a finite partition of the form 
$m = u_0 < u_1 < \cdots < u_L = M$ so that for 
$0 \leq j \leq L-1$ we have
\begin{equation*}
    F(u-_{j+1}) - F(u_j) \leq \epsilon.
\end{equation*} 
\end{l1}

\noindent
\textbf{Proof of \ref{lemma1gl}:} For $j \geq 0$; define
\begin{equation}
    u_{j+1} = \mathrm{Sup}\{v: F(v) \leq F(u_j)+\epsilon \} .
    \label{l1eq}
\end{equation} 
From Eqn \eqref{l1eq}, we have
$F(u_{j+1}) \geq F(u_j)+\epsilon$. Thus the 
height of jump between $F(u_{j+1})$ and $F(u_j)$ 
is at least $\epsilon$. This can happen only a finite number of times and so we get the partition 
of the desired form. $\quad \square$

\noindent
\textbf{Proof of Clause \#1.} For $\epsilon > 0$, we 
must find $N = N(\epsilon)$ so that for $n > N$ 
the event
\begin{eqnarray}
    \sup_{y_i} |F_{n}(y_i)-F(y_i)| < \epsilon
    \label{clause1event}
\end{eqnarray}
has probability going to one.
Consider a finite partition of $[m, M]$
as in Lemma \ref{lemma1gl} giving
\begin{equation}
    F(u^-_{j+1}) - F(u_j) \leq \frac{\epsilon}{2} .\label{eps}
\end{equation}
for $0 \leq j \leq L-1$.
Now, for $y_i \in [m, M]$ there is a $j$ so that 
$u_j \leq y_i \leq u_{j+1}$.  So, we have
      \begin{equation}
          F_{n}(u_j) \leq F_{n}(y_i) \leq F_{n}(u^-_{j+1}) \label{l2fn}
      \end{equation} 
and
      \begin{equation}
          F(u_j) \leq F(y_i) \leq F(u^-_{j+1}).
          \label{l2f}
      \end{equation} 
Together,\eqref{l2f} and \eqref{l2fn} imply
      \begin{equation*}
          F_{n}(u_j)- F(u^-_{j+1})\leq F_{n}(y_i)-F(y_i) \leq F_{n}(u^-_{j+1}) - F(u_j).
      \end{equation*}
Re-arranging gives
      \begin{equation}
         F_{n}(u_j)-F(u_j)+F(u_j)- F(u^-_{j+1})
         \leq F_{n}(y_i)-F(y_i)  
         \label{adj1}
      \end{equation}
and
      \begin{equation}
          F_{n}(u^-_{j+1}) - F(u^-_{j+1})+F(u^-_{j+1})- F(u_j) 
          \geq F_{n}(y_i)-F(y_i).
          \label{adj2}
      \end{equation}
Using \eqref{eps} in \eqref{adj1} and \eqref{adj2} 
we have
      \begin{center}
     $F_{n}(u_j)-F(u_j)-\frac{\epsilon}{2} 
     \leq F_{n}(y_i)-F(y_i)$
     
    $F_{n}(u^-_{j+1}) - F(u^-_{j+1})+\frac{\epsilon}{2} \geq F_{n}(y_i)-F(y_i)$.
\end{center}
For each $j$, let $N_j = N_j(\epsilon)$ be such that 
for $n > N_j$ we have
$F_{n}(u_j)-F(u_j) > -\frac{\epsilon}{2}$
and let $M_j = M_j(\epsilon)$ be such that 
for $n > M_j$   we have
  \begin{center}
        $F_{n}(u^-_j) - F(u^-_j) < \frac{\epsilon}{2}$.
  \end{center}
Write $N= \underset{1\leq j\leq K}{\mathrm{max}} \{N_j,M_j\}$. Then, for $n > N$ and for any $x \in [m,M]$
we get that \eqref{clause1event} has probability
going to zero.  Hence, Clause \#1 follows.

\medskip

The proof of Clause II requires the following.

\medskip

\begin{l1}
\label{lemma2gl}
Suppose $F_{n}$ and $F$ are DF's on $[m,M]$
so that $F_{n}(y_i) \longrightarrow F(y_i)$ pointwise in
probability for $y_i \in \mathbb{Q}$. 
Suppose also that the jump points of $F$ are well-behaved,
i.e., $F_{n}(y_i)-F_{n}(y_i^-) \longrightarrow F(y_i)-F(y_i^-)$ 
pointwise in probability
for all jump points of $F$. 
Then, for all $y_i$, we have
\begin{equation*}
    F_{n}(y_i) \longrightarrow F(y_i) 
\end{equation*} and 
\begin{equation*}
   F_{n}(y_i^-) \longrightarrow F(y_i^-). 
\end{equation*}
in the probability on the data.
\end{l1}

\noindent
\textbf{Proof of Lemma \ref{lemma2gl}: }
For any $y_i \in (0,M]$ and $w, z \in \mathbb{Q}$ 
with $w < y_i < z$. We have
\begin{center}
    $F_{n}(w) \leq F_{n}(y_i) \leq F_{n}(z)$.
\end{center}
So, if $y_i$ is a continuity point we have
\begin{center}
    $F(w) \leq \underset{n \longrightarrow \infty}{ess. \liminf} F_{n}(y_i) \leq \underset{n \longrightarrow \infty}{ess. \limsup} F_{n}(y_i) \leq F(z)$.
\end{center}  
and the lemma follows.
Likewise, if $y_i$ is a jump point of $F$ we have
\begin{center}
    $F_{n}(w)+F_{n}(y_i)-F_{n}(y_i^-) \leq F_{n}(y_i) \leq F_{n}(z)$
\end{center}
and hence taking convergences in the distribution 
for the data
\begin{center}
$F(w)+F(y_i)-F(y_i^-) 
\leq 
\underset{n \longrightarrow \infty}{ess \liminf} F_{n}(y_i) 
\leq 
\underset{n \longrightarrow \infty}{ess \limsup} F_{n}(y_i) 
\leq F(z)$.
\end{center}
Since $\underset{w\rightarrow y_i^-} \lim F(w) = F(y_i^-)$ and
$\underset{z\rightarrow y_i^+}\lim F(w) = F(y_i)$ the first
part follows.

To show $F_{n}(y_i^-) \longrightarrow F(y_i^-)$ in 
probability,
let $x$ be a continuity point of $F$. Since, $F_{n}(y_i^-) \leq F(y_i)$, it follows that
\begin{center}
    $\underset{n \longrightarrow \infty}{ess \limsup} F_{n}(y_i^-) \leq \underset{t \longrightarrow \infty}{ess \limsup} F_{n}(y_i) \leq F(y_i) \leq F(y_i^-)$.
\end{center}
Again, for any $w \in \mathbb{Q}$ with $w < y_i$, we have $F_{n}(w) \leq F_{n}(y_i^-)$ which implies that
\begin{center}
    $F(w) \leq \underset{n \longrightarrow \infty}{ess \liminf} F_{n}(y_i^-)$.
\end{center}
Since $\underset{w \longrightarrow y_i^-} {\lim} F(w) 
= F(y_i^-)$ the statement follows.
Finally, support again that $y_i$ is a jump point of 
$F$. By assumption $F_{n}(y_i)-F_{n}(y_i^-) \longrightarrow F(y_i)-F(y_i^-)$ and it has been shown that $F_{n}(y_i) \longrightarrow F(y_i)$.
Thus, $F_{n}(y_i^-) \longrightarrow F(y_i^-)$. 
$\quad \square$

\medskip

\textbf{Proof of Clause \#2.} We must show that there 
is a set $N$ with $P\{N\} = 0$ so that for all 
$w \notin N$ 
(i) $\hat{F}_{n}(y_i) \longrightarrow F(y_i)$ for
$y_i \in \mathbb{Q}$ and (ii) $\hat{F}_{n}(y_i)-\hat{F}_{n}(y_i^-)\longrightarrow F(y_i)-F(y_i^-)$ for all jump points of $F$, both in the
joint probabilities of $P_\mathcal{H}$ of $H$ and $P$ of the data where $P$ has DF $F$.

To do this, recall that we have $\hat{F}_{n}(y_i) \longrightarrow F(y_i)$ in probability from Cor. \ref{cor1}. 
Since,$\hat{F}_{n}(y_i)$ is bounded that will imply 
$\mathcal{L}_1$ convergence which again implies 
convergence almost surely. 
So, for each $y_i \in \mathbb{Q}$, let $N_{y_i}$ be a set satisfying $P\{N_{y_i}\} = 0$ 
and $\hat{F}_{n}(\omega; y_i) \longrightarrow F(y_i)$ for all $\omega \notin N_{y_i}$. Define $N_1 = \underset{y_i \in \mathbb{Q}}{\bigcup}N_{y_i} $. 
By construction, for all $\omega \notin N_1$, 
$\hat{F}_n (\omega; y_i) \longrightarrow F(y_i)$. Since, $\mathbb{Q}$ is countable, $P\{N_1\} = 0$.

Next, define $J_i$ as the set of jump points of 
$F$ of size at least $1/i$; for $i \geq 1$. For 
each $i, J_i$ is finite.  Now, 
$J = \underset{1 \leq i \leq \infty}{\bigcup}J_i $ 
is the set of all jump points of $F$. 
Let $M_{y_i}$ satisfy $P\{M_{y_i}\} = 0$ so that for all 
$\omega \notin M_{y_i}$ we have 
$\hat{F}_{n}(\omega; y_i)-\hat{F}_{n}(\omega; y_i^-) \longrightarrow F(y_i)-F(y_i^-)$ in probability.
Let $N_2 = \underset{y_i \in J}{\bigcup}M_{y_i}$. 
Since $J$ is countable, $P\{N_2\} = 0$. 

To finish,
set $N = N_1 \bigcup N_2$ so that by construction 
for $\omega \notin N$, (i) and (ii) hold. Hence
Clause I and Lemma \ref{lemma2gl} give Clause II.

\medskip

\textbf{Proof of Clause \#3.}This proof is similar 
to the proof of Clause \#2. 
First note that Cor. \ref{cor1}
gives $\hat{F}_{n}(y_i) \longrightarrow F_{n}(y_i)$ in probability. Now, since $\hat{F}_{n}(y_i)$ is a bounded function, convergence in probability implies convergence in $L_1$ norm and, hence, convergence almost sure.

To strengthen the mode of convergence to 
Kolmogorov-Smirnov distance, begin by defining 
$T_{y_i}$ to be a set such that $P\{T_{y_i}\} = 0$ and 
for all $\omega \notin T_{y_i}$, 
$\hat{F}_{n}(y_i) \longrightarrow F_{n}(y_i)$. 
Also define $T_1 = \underset{y_i \in \mathbb{Q}}{\bigcup}T_{y_i} $. As before, $P\{T_1\} = 0$. 
Thus, for all $\omega \in T_1$, 
$\hat{F}_{n}(\omega; y_i) \longrightarrow F_{n}(y_i)$. 

Again as before, let $J_i$ be the set of jump 
points of $F_{n}$ of size at least $1/i$ so
that $J_i$ is finite for each $i$ and let 
$J =  \underset{1 \leq i \leq \infty}{\bigcup}J_i$. Now, for each $y_i \in J$, let $M_{y_i}$ denote a set such that 
$P\{M_{y_i}\} = 0$ and for all $\omega \notin M_{y_i}$, $\hat{F}_{n}(y_i)-\hat{F}_{n}(y_i^-) \longrightarrow F_{n}(y_i)-F_{n}(y_i^-)$. Also, let 
$T_2 = \underset{y_i \in J}{\bigcup}M_{y_i}$. Since, 
$J$ is countable, $P\{T_2\} = 0$. 

To finish, let $T = T_1 \bigcup T_2$. 
By construction, for $\omega \notin T$, 
$\hat{F}_{n}(y_i) \longrightarrow F_{n}(y_i)$ and 
$\hat{F}_{n}(y_i)-\hat{F}_{n}(y_i^-) \longrightarrow F_{n}(y_i)-F_{n}(y_i^-)$. So, by Clause I of the Theorem and by Lemma \ref{lemma2gl} Clause III follows. $\quad \square$
\end{proof}

\subsection{Proof of Theorem \ref{studentt}}

\begin{proof}
We use $p$ generically to indicate probability 
densities. We
use $w$ when we want to emphasize that a density is a prior or posterior and $m$ to emphasize that a density is a mixture of densities for its indicated arguments. Now, the posterior density for $a^n,\sigma^2|y^n$ is given by:
\begin{eqnarray}
  p(a^n,\sigma^2|y^n) &=& \frac{\mathcal{L}_1(a, \sigma^2|y^n)\times w(a^n,\sigma^2)}{m(y^n)}
  = \frac{p(y^n,a^n,\sigma^2)}{m(y^n)}.
  \label{closedpost}
\end{eqnarray}. 
We know that
\begin{eqnarray}
    Y^n \sim \mathcal{N}(a^n,\sigma^2(I_{n \times n}+K_{n \times n}))
    \label{likey}
\end{eqnarray} 
and
\begin{eqnarray}
     w(a^n,\sigma^2) &=& \mathcal{N}(\gamma1^n,\sigma^2\delta^2 I_{n \times n})\textbf{ }\mathcal{IG}(\alpha,\beta)
    \nonumber\\
    &=&   \frac{e^{-\frac{1}{2\sigma^2}(a^n-\gamma1^n)'(\delta^2I_{n \times n})^{-1}(a^n-\gamma1^n)}}{(2\pi)^\frac{n}{2}(\sigma^2\delta^2)^\frac{n}{2}}\frac{\beta^\alpha}{\Gamma(\alpha)}\left(\frac{1}{\sigma^2}\right)^{\alpha+1}e^{-\frac{\beta}{\sigma^2}} .
    \label{jtprior}
\end{eqnarray}
From \eqref{likey} and \eqref{jtprior} the numerator in \eqref{closedpost} is given by
\begin{eqnarray}
   p(a^n,\sigma^2|y^n) &=& \frac{e^{-\frac{1}{2\sigma^2}(a^n-\gamma 1^n)'(\delta^2I_{n \times n})^{-1}(a^n-\gamma 1^n)}}{(2\pi)^\frac{n}{2}(\sigma^2\delta^2)^\frac{n}{2}}
   \nonumber \\
    &\times& \frac{\beta^\alpha}{\Gamma(\alpha)}\left(\frac{1}{\sigma^2}\right)^{\alpha+1}e^{-\frac{\beta}{\sigma^2}}\frac{e^{-\frac{1}{2\sigma^2}(y^n-a^n)'(I+K)_{n \times n}^{-1}(y^n-a^n)}}{(2\pi)^\frac{n}{2}(\sigma^2)^\frac{n}{2}|(I+K)_{n \times n}|^\frac{1}{2}}
     \nonumber\\
     &=& \frac{\beta^\alpha}{(2\pi)^{\frac{n}{2}+\frac{n}{2}}|(I+K)_{n \times n}|^{\frac{1}{2}}(\delta^2)^\frac{n}{2}\Gamma(\alpha)}
     \nonumber \\
     &\times& e^{-\frac{1}{\sigma^2}[\beta+\frac{1}{2}(y^n-a^n)'(I+K)_{n \times n}^{-1}(y^n-a^n)+\frac{1}{2}(a^n-\gamma 1^n)'(\delta^2I_{n \times n})^{-1}(a^n-\gamma 1^n)]} \nonumber\\
     && \times \left(\frac{1}{\sigma^2}\right)^{\alpha+1+\frac{n}{2}+\frac{n}{2}}.
     \nonumber\\
     \label{proppost}
\end{eqnarray}
We simplify the terms in the exponent in \eqref{proppost} as follows. It is
\begin{eqnarray*}
 && 
\beta +\frac{1}{2}(y^n-a^n)'(I+K)_{n \times n}^{-1}(y^n-a^n)+\frac{1}{2}(a^n-\gamma1^n)'(\delta^2I_{n \times n})^{-1}(a^n-\gamma 1^n)
\nonumber\\
&=& \beta + \frac{1}{2}y'^n(I+K)_{n \times n}^{-1} y^n - \frac{1}{2}a'^n(I+K)_{n \times n}^{-1}y^n - \frac{1}{2}y'^n(I+K)_{n \times n}^{-1} a^n + \frac{1}{2}a'^n(I+K)_{n \times n}^{-1}a^n\nonumber\\ 
&&+ \frac{1}{2}a'^n (\delta^2I_{n \times n})^{-1}a^n  -\frac{1}{2}\gamma 1'^n(\delta^2I_{n \times n})^{-1}a^n -\frac{1}{2}\gamma a'^n(\delta^2)^{-1}1^n +\frac{1}{2}\gamma^2 1'^n(\delta^2I_{n \times n})^{-1}1^n
\nonumber\\
&=& \beta + \frac{1}{2}a'^n[(I+K)_{n \times n}^{-1}+(\delta^2I_{n \times n})^{-1}]a^n -a'^n[(I+K)_{n \times n}^{-1}y^n + \gamma(\delta^2I_{n \times n})^{-1}1^n] 
\nonumber \\
&& + \frac{1}{2}y'^n(I+K)_{n \times n}^{-1}y^n + 
\frac{1}{2}\gamma^2 1'^n(\delta^2I_{n \times n})^{-1}1^n
\nonumber\\ 
&=& \beta + \frac{1}{2}a'^n[\{(I+K)_{n \times n}^{-1}+(\delta^2I_{n \times n})^{-1}\}^{-1}]^{-1}a^n
\nonumber\\
&& - a'^n[\{(I+K)_{n \times n}^{-1}+(\delta^2I_{n \times n})^{-1}\}^{-1}]^{-1}[(I+K)_{n \times n}^{-1}+(\delta^2I_{n \times n})^{-1}]^{-1}
\nonumber\\
&&[(I+K)_{n \times n}^{-1}y^n+\gamma(\delta^2I_{n \times n})^{-1}1^n]+ \frac{1}{2}y'^n(I+K)_{n \times n}^{-1}y^n + \frac{1}{2}\gamma^2 1'^n(\delta^2I_{n \times n})^{-1}1^n .
\end{eqnarray*}
So
\begin{eqnarray}
&& w(a^n,\sigma^2|y^n) 
\nonumber \\
&=& \frac{\beta^\alpha}{(2\pi)^{\frac{n}{2}+\frac{n}{2}}|(I+K)_{n \times n}|^{\frac{1}{2}}(\delta^2)^\frac{n}{2}\Gamma(\alpha)}\left(\frac{1}{\sigma^2}\right)^{\alpha+1+\frac{n}{2}+\frac{n}{2}} \cdot
 \nonumber\\
 && \times e^{-\frac{1}{\sigma^2}[\beta + \frac{1}{2}a'^n[\{(I+K)_{n \times n}^{-1}+(\delta^2I_{n \times n})^{-1}\}^{-1}]^{-1}a^n]}\nonumber\\
 && \times e^{-\frac{1}{\sigma^2}[-a'^n[\{(I+K)_{n \times n}^{-1}+(\delta^2I_{n \times n})^{-1}\}^{-1}]^{-1}[(I+K)_{n \times n}^{-1}+(\delta^2I_{n \times n})^{-1}]^{-1}][(I+K)_{n \times n}^{-1}y^n+\gamma(\delta^2I_{n \times n})^{-1}1^n]}
 \nonumber\\
 && \times e^{-\frac{1}{\sigma^2}\{ \frac{1}{2}y'^n(I+K)_{n \times n}^{-1}y^n + \frac{1}{2}\gamma^2 1'^n(\delta^2I_{n \times n})^{-1}1^n]\}}
 \nonumber\\
 &=&  \frac{\beta^\alpha}{(2\pi)^{\frac{n}{2}+\frac{n}{2}}|(I+K)_{n \times n}|^{\frac{1}{2}}(\delta^2)^\frac{n}{2}\Gamma(\alpha)}
 \nonumber \\
 && \times \left(\frac{1}{\sigma^2}\right)^{\alpha+1+\frac{n}{2}+\frac{n}{2}} e^{-\frac{1}{\sigma^2}\{ \beta+\frac{1}{2}y'^n(I+K)_{n \times n}^{-1}y^n + \frac{1}{2}\gamma^2 1'^n(\delta^2I_{n \times n})^{-1}1^n]\}}
 \nonumber\\
  && \times e^{-\frac{1}{\sigma^2}[\frac{1}{2}a'^n[\{(I+K)_{n \times n}^{-1}+(\delta^2I_{n \times n})^{-1}\}^{-1}]^{-1}a^n} \nonumber\\
 && \times e^{-\frac{1}{\sigma^2} [- a'^n[\{(I+K)_{n \times n}^{-1}+(\delta^2I_{n \times n})^{-1}\}^{-1}]^{-1}[(I+K)_{n \times n}^{-1}+(\delta^2I_{n \times n})^{-1}]^{-1}][(I+K)_{n \times n}^{-1}y+\gamma(\delta^2I_{n \times n})^{-1}1^n]}.
 \label{exponent}
\end{eqnarray}
So, if we set
\begin{eqnarray}
\label{V}
    V_{n \times n} &=& [(I+K)_{n \times n} ^{-1}+(\delta^2I_{n \times n})^{-1}]^{-1}\\
    \mu &=& [(I+K)_{n \times n}^{-1}+(\delta^2I_{n \times n})^{-1}]^{-1}[(I+K)_{n \times n}^{-1}y^n +\gamma (\delta^2I_{n \times n})^{-1}1^n]
    \nonumber\\
      \label{mu}
    &=& V_{n \times n}[(I+K)_{n \times n}^{-1}y^n +\gamma (\delta^2I_{n \times n})^{-1}1^n]\\
    \label{betastar}
    \beta^* &=& \beta + \frac{1}{2}y'^n(I+K)_{n \times n}^{-1}y^n + \frac{1}{2}\gamma^2 1^n\delta^2I_{n \times n}]^{-1}1^n - \frac{1}{2} \mu'^n V_{n \times n}^{-1} \mu^n \\
    \alpha^* &=& n+\alpha,
    \nonumber\\
    \label{alphastar}
\end{eqnarray}
the expression in \eqref{exponent} becomes
   \begin{eqnarray*}
w(a^n,\sigma^2|y^n) &=& \frac{\beta^\alpha}{(2\pi)^n|(I+K)_{n \times n}|^{\frac{1}{2}}(\delta^2)^\frac{n}{2}\Gamma(\alpha)}
\nonumber \\
&& \times \left(\frac{1}{\sigma^2}\right)^{\alpha^*+1}e^{-\frac{1}{\sigma^2}(\beta^*+\frac{1}{2}\mu'^nV_{n \times n}^{-1}\mu^n+\frac{1}{2}a'^nV_{n \times n}^{-1}a^n-a'^nV_{n \times n}^{-1}\mu_{n \times n})} \cdot
\nonumber\\
&=& \frac{\beta^\alpha}{(2\pi)^n|I+K|^{\frac{1}{2}}(\delta^2)^\frac{n}{2}\Gamma(\alpha)}\left(\frac{1}{\sigma^2}\right)^{\alpha^*+1}e^{-\frac{1}{\sigma^2}[\frac{1}{2}(a^n-\mu^n)'V_{n \times n}^{-1}(a^n-\mu^n)]} e^{-\frac{1}{\sigma^2}\beta^*}.
\end{eqnarray*} 
The denominator in \eqref{closedpost} is given by:
\begin{eqnarray}
  m(y^n) &=& \int_{\mathbf{R}^+}^{} \int_{\mathbf{R}^n}^{}\mathcal{L}_1(a^n, \sigma^2|y^n)\times w(a^n,\sigma^2) \rm{d} a^n d\sigma^2.\nonumber\\
  &=& \int_{\mathbf{R}^+}^{} \Biggl[\int_{\mathbf{R}^n}^{} \frac{1}{(2\pi)^{\frac{n}{2}}(\sigma^2)^{\frac{n}{2}}|(I+K)_{n \times n}|^\frac{1}{2}} \times \frac{1}{(2\pi)^{\frac{n}{2}}(\sigma^2\delta^2)^\frac{n}{2}} \nonumber\\
 &&  \times e^{-\frac{1}{2\sigma^2}(y^n-a^n)^{'}(I+K)_{n \times n}^{-1}(y^n-a^n)} e^{-\frac{1}{2\sigma^2}(a^n-\gamma 1^n)^{'}(\delta^2I_{n \times n})^{-1}(a^n-\gamma 1^n)}\rm{d}a^n\Biggr] 
 \nonumber \\
&& \times \frac{\beta^\alpha}{\Gamma(\alpha)}\Bigl(\frac{1}{\sigma^2}\Bigr)^{\alpha+1}e^{-\frac{\beta}{\sigma^2}}d\sigma^2
  \nonumber\\
  &=& \int_{\mathbf{R}^+}\Biggl[\frac{1}{(2\pi)^n(\sigma^2)^n|(I+K)_{n \times n}|^\frac{1}{2}(\delta^2)^\frac{n}{2}} 
  \nonumber \\
  &&\times \int_{\mathbf{R^n}}^{} \Bigl\{e^{-\frac{1}{2\sigma^2}[y^{'n}(I+K)_{n \times n}^{-1}y-a^{'n}(I+K)_{n \times n}^{-1}y^n - y^{'n}(I+K)_{n \times n}^{-1}a^n +a^{'n}(I+K)_{n \times n}^{-1}a^n]}
  \nonumber \\
  && \times
 e^{-\frac{1}{2\sigma^2}[a^{'n}(\delta^2I)^{-1}a^n-\gamma1^{'n}(\delta^2I_{n \times n})^{-1}1^n -\gamma a^{'n}(\delta^2I_{n \times n})^{-1}1^n +\gamma^21^{'n}(\delta^2I_{n \times n})^{-1}1^n]}\Bigr\} \rm{d}a^n
 \nonumber \\
 &&\times \frac{\beta^\alpha}{\Gamma(\alpha)}\Bigl(\frac{1}{\sigma^2}\Bigr)^{\alpha+1}e^{-\frac{\beta}{\sigma^2}}\Biggr]d\sigma^2
 \nonumber \\
 &=& \int_{\mathbf{R}^+}^{}\Biggl(\frac{e^{-\frac{1}{2\sigma^2}[y^{'n}(I+K)_{n \times n}^{-1}y^n+\gamma^21^{'n}(\delta^2I_{n \times n})^{-1}1^n]}}{(2\pi)^n(\sigma^2)^n|(I+K)_{n \times n}|^\frac{1}{2}(\delta^2)^{\frac{n}{2}}} \times
\nonumber\\
&& \Biggl[\int_{\mathbf{R^n}}^{}e^{-\frac{1}{2\sigma^2}\Bigl[a^{'n}\Bigl\{(I+K)_{n \times n}^{-1}+(\delta^2I_{n \times n})^{-1}\Bigr\}a^n -2a^{'n}\Bigl\{(I+K)_{n \times n}^{-1}+(\delta^2I_{n \times n})^{-1}\Bigr\}\Bigl\{(I+K)_{n \times n}^{-1}+(\delta^2I_{n \times n})^{-1}\Bigr\}^{-1}\Bigr]}\nonumber\\
&& \times e^{-\frac{1}{2\sigma^2}\Bigl\{(I+K)_{n \times n}^{-1}y^n+\gamma(\delta^2I_{n \times n})^{-1}1^n\Bigr\}\Bigr]}\rm{d}a^n\Biggr]
\times \frac{\beta^\alpha}{\Gamma(\alpha)}\Bigl(\frac{1}{\sigma^2}\Bigr)^{\alpha+1}e^{-\frac{\beta}{\sigma^2}}\Biggr)d\sigma^2 .
\label{above1}
\end{eqnarray}

Rewriting \eqref{above1} in terms of equations \eqref{V} to \eqref{alphastar} gives
\begin{eqnarray}
\label{incomplpsq}
&=& \int_{\mathbf{R}^+}^{}\Biggl[\frac{e^{-\frac{1}{2\sigma^2}[y^{'n}(I+K)_{n \times n}^{-1}y^n+\gamma^21^{'n}(\delta^2I_{n \times n})^{-1}1^n]}}{(2\pi)^n(\sigma^2)^n|(I+K)_{n \times n}|^\frac{1}{2}(\delta^2)^{\frac{n}{2}}} 
\nonumber \\
&& \times
\Biggl(\int_{\mathbf{R^n}}^{} e^{-\frac{1}{2\sigma^2}\Bigl[a^{'}V_{n \times n}^{-1}a^n - 2a^{'n}V_{n \times n}^{-1}\mu^n\Bigr]} \rm{d}a^n\Biggr) 
\times \frac{\beta^\alpha}{\Gamma(\alpha)}\Bigl(\frac{1}{\sigma^2}\Bigr)^{\alpha+1}e^{-\frac{\beta}{\sigma^2}}\Biggr]d\sigma^2.
\nonumber\\
\end{eqnarray}
We complete the square in the inner integral (with respect to $a^n$) by multiplying and dividing \eqref{incomplpsq} by  $e^{-\frac{1}{2\sigma^2}}\mu^{'n}V_{n \times n}^{-1}\mu^n$. This gives
\begin{eqnarray}
m(y^n)&=& \int_{\mathbf{R}^+}^{}\Biggl[\frac{e^{-\frac{1}{2\sigma^2}[y^{'n}(I+K)_{n \times n}^{-1}y^n+\gamma^21^{'n}(\delta^2I_{n \times n})^{-1}1^n]}}{(2\pi)^n(\sigma^2)^n|(I+K)_{n \times n}|^\frac{1}{2}(\delta^2)^{\frac{n}{2}}} e^{\frac{1}{2\sigma^2}\mu^{'n}V_{n \times n}^{-1}\mu^n} 
\nonumber \\
&& \times \Biggl(\int_{\mathbf{R^n}}^{} e^{-\frac{1}{2\sigma^2}(a^n-\mu^n)^{'}V_{n \times n}^{-1}(a^n-\mu^n)} \rm{d}a^n\Biggr)
\frac{\beta^\alpha}{\Gamma(\alpha)}\Bigl(\frac{1}{\sigma^2}\Bigr)^{\alpha+1}e^{-\frac{\beta}{\sigma^2}}\Biggr]d\sigma^2.
\nonumber\\
\label{complsq}
\end{eqnarray}
The integral with respect to $a^n$ in \eqref{complsq} becomes $1$ if we divide and multiply \eqref{compsq} by $(2\pi)^{\frac{n}{2}}(\sigma^2)^\frac{n}{2}|V_{n \times n}|^\frac{1}{2}$, i.e.,
\begin{eqnarray}
m(y^n)&=& \int_{\mathbf{R}^+}^{}\Biggl[\frac{e^{-\frac{1}{2\sigma^2}[y^{'n}(I+K)_{n \times n}^{-1}y^n+\gamma^21^{'n}(\delta^2I_{n \times n})^{-1}1^n]}}{(2\pi)^n(\sigma^2)^n|(I+K)_{n \times n}|^\frac{1}{2}(\delta^2)^{\frac{n}{2}}} (2\pi)^{\frac{n}{2}}(\sigma^2)^\frac{n}{2}|V_{n \times n}|^\frac{1}{2} e^{\frac{1}{2\sigma^2}\mu^{'n}V_{n \times n}^{-1}\mu^n} \times
\nonumber\\
&&\Biggl(\frac{1}{(2\pi)^{\frac{n}{2}}(\sigma^2)^\frac{n}{2}|V_{n \times n}|^\frac{1}{2}}\int_{\mathbf{R}}^{} e^{-\frac{1}{2\sigma^2}(a^n-\mu^n)^{'}V_{n \times n}^{-1}(a^n-\mu^n)} \rm{d}a^n\Biggr)\times \frac{\beta^\alpha}{\Gamma(\alpha)}\Bigl(\frac{1}{\sigma^2}\Bigr)^{\alpha+1}e^{-\frac{\beta}{\sigma^2}}\Biggr]\rm{d}\sigma^2
\nonumber\\
&=&\int_{\mathbf{R}^+}^{}\Biggl[\frac{(2\pi)^{\frac{n}{2}}(\sigma^2)^\frac{n}{2}|V_{n \times n}|^\frac{1}{2}}{(2\pi)^n (\sigma^2)^n |(I+K)_{n \times n}|^{\frac{1}{2}}(\delta^2)^{\frac{n}{2}}} e^{\frac{1}{2\sigma^2}\mu^{'n}V_{n \times n}^{-1}\mu^n} e^{-\frac{1}{2\sigma^2}[y^{'n}(I+K)_{n \times n}^{-1}y^n+\gamma^21^{'n}(\delta^2I_{n \times n})^{-1}1^n]}\nonumber\\
&& \times \frac{\beta^\alpha}{\Gamma(\alpha)}\Bigl(\frac{1}{\sigma^2}\Bigr)^{\alpha+1}e^{-\frac{\beta}{\sigma^2}}\Biggr]\rm{d}\sigma^2.
\label{marginalwrtsigmasq}
\end{eqnarray}
Rearranging \eqref{marginalwrtsigmasq} gives
\begin{eqnarray}
m(y^n)&=& \frac{|V_{n \times n}|^{\frac{1}{2}}}{(2\pi \delta^2)^{\frac{n}{2}}|(I+K)_{n \times n}|^{\frac{1}{2}}} \frac{\beta^\alpha}{\Gamma(\alpha)} 
\nonumber \\
&& \times \int_{\mathbf{R}^+}^{} \Bigl(\frac{1}{\sigma^2}\Bigr)^{\alpha+\frac{n}{2}+1} e^{-\frac{1}{\sigma^2}\Bigl[\beta+\frac{1}{2}\Bigl\{ y^{'n}(I+K)_{n \times n}^{-1}y^n+\gamma^2 1^{'n}(\delta^2I_{n \times n})^{-1}1^n-\mu'V_{n \times n}^{-1}\mu^n\Bigr\}\Bigr ]} \rm{d}\sigma^2.
\label{transbetastar}
\nonumber\\
\end{eqnarray}
Recall from \eqref{betastar} and \eqref{alphastar} that: 
\begin{eqnarray*}
    \beta^* &=& \beta + \frac{1}{2}y'^n(I+K)_{n \times n}^{-1}y^n + \frac{1}{2}\gamma^2 1^n\delta^2I_{n \times n}]^{-1}1^n - \frac{1}{2} \mu'^n V_{n \times n}^{-1} \mu^n \\
    \alpha^* &=& n+\alpha
\end{eqnarray*}. Using them in \eqref{transbetastar} gives
\begin{eqnarray}
    m(y^n) &=&   \frac{|V_{n \times n}|^{\frac{1}{2}}}{(2\pi \delta^2)^{\frac{n}{2}}|(I+K)_{n \times n}|^{\frac{1}{2}}} \frac{\beta^\alpha}{\Gamma(\alpha)} \int_{\mathbf{R}^+}^{} \Bigl(\frac{1}{\sigma^2}\Bigr)^{\alpha^*-\frac{n}{2}+1} e^{-\frac{1}{\sigma^2}\beta^*} \rm{d}\sigma^2.
    \label{invgammapdf}
\end{eqnarray}
The integrand in \eqref{invgammapdf} will be the pdf of an Inverse Gamma distribution and the integral will be 1, if we multiply and divide \eqref{invgammapdf} by $\frac{\beta^{*^{\alpha^*-\frac{n}{2}}}}{\Gamma(\alpha^*-\frac{n}{2})}.$ So we have
\begin{eqnarray}
    m(y^n) &=& \frac{|V_{n \times n}|^{\frac{1}{2}}}{(2\pi \delta^2)^{\frac{n}{2}}|I+K|^{\frac{1}{2}}} \frac{\beta^\alpha}{\Gamma(\alpha)} \frac{\Gamma(\alpha^*-\frac{n}{2})}{\beta^{*^{\alpha^*-\frac{n}{2}}}}\int_{\mathbf{R}^+}^{}\frac{\beta^{*^{\alpha^*-\frac{n}{2}}}}{\Gamma(\alpha^*-\frac{n}{2})} \Bigl(\frac{1}{\sigma^2}\Bigr)^{\alpha^*-\frac{n}{2}+1} e^{-\frac{1}{\sigma^2}\beta^*} \rm{d}\sigma^2.
    \nonumber\\
    &=& \frac{|V_{n \times n}|^{\frac{1}{2}}}{(2\pi \delta^2)^{\frac{n}{2}}|I+K|^{\frac{1}{2}}} \frac{\beta^\alpha}{\Gamma(\alpha)} \frac{\Gamma(\alpha^*-\frac{n}{2})}{\beta^{*^{\alpha^*-\frac{n}{2}}}}.
    \label{marginaly}
\end{eqnarray}

Using \eqref{alphastar} in \eqref{marginaly} for $n+1$ and $n$ gives
\begin{eqnarray}
\label{cbetastar}
    \frac{m(y^{n+1})}{m(y^n)} &=& \frac{\frac{|V_{n+1 \times n+1}|^{\frac{1}{2}}}{(2\pi \delta^2)^{\frac{n+1}{2}}|(I+K)_{n+1 \times n+1}|^{\frac{1}{2}}} \frac{\beta^\alpha}{\Gamma(\alpha)} \frac{\Gamma(\alpha+\frac{n+1}{2})}{\beta_{n+1}^{*^{\alpha+\frac{n+1}{2}}}}}{\frac{|V_{n \times n}|^{\frac{1}{2}}}{(2\pi \delta^2)^{\frac{n}{2}}|(I+K)_{n \times n}|^{\frac{1}{2}}} \frac{\beta^\alpha}{\Gamma(\alpha)} \frac{\Gamma(\alpha+\frac{n}{2})}{\beta_n^{*^{\alpha+\frac{n}{2}}}}}
    \nonumber\\
    &=& \frac{|V_{n+1 \times n+1}|^\frac{1}{2}}{|V_{n \times n}|^\frac{1}{2}} \frac{|(I+K)_{n \times n}|^\frac{1}{2}}{|(I+K)_{n+1 \times n+1}|^\frac{1}{2}} \frac{(2\pi\delta^2)^\frac{n}{2}}{(2\pi\delta^2)^\frac{n+1}{2}} \frac{\frac{\Gamma(\alpha+\frac{n+1}{2})}{\Gamma(\alpha)}}{\frac{\Gamma(\alpha+\frac{n}{2})}{\Gamma(\alpha)}} \frac{\beta^\alpha}{\beta^\alpha} \frac{(\beta_{n+1}^*)^{-(\alpha+\frac{n+1}{2})}}{(\beta_{n}^*)^{-(\alpha+\frac{n}{2})}}
    \nonumber\\
    &=& c \frac{(\beta_{n+1}^*)^{-(\alpha+\frac{n+1}{2})}}{(\beta_{n}^*)^{-(\alpha+\frac{n}{2})}},
\end{eqnarray} where
\begin{eqnarray}
\label{c}
    c &=& \frac{|V_{n+1 \times n+1}|^\frac{1}{2}}{|V_{n \times n}|^\frac{1}{2}} \frac{|(I+K)_{n \times n}|^\frac{1}{2}}{|(I+K)_{n+1 \times n+1}|^\frac{1}{2}} \frac{(2\pi\delta^2)^\frac{n}{2}}{(2\pi\delta^2)^\frac{n+1}{2}} \frac{\frac{\Gamma(\alpha+\frac{n+1}{2})}{\Gamma(\alpha)}}{\frac{\Gamma(\alpha+\frac{n}{2})}{\Gamma(\alpha)}} \frac{\beta^\alpha}{\beta^\alpha}.
\end{eqnarray}
From \eqref{mu} and \eqref{betastar}, we have
\begin{eqnarray*}
    \mu^n &=& V_{n \times n}[(I+K)_{n \times n}^{-1}y_n +\gamma (\delta^2I)^{-1}1_n]\\
    \label{betastarn}
    \beta^*_n &=& \beta + \frac{1}{2}y_n'(I+K)_{n \times n}^{-1}y + \frac{1}{2}\gamma^2 1_n[\delta^2I]^{-1}1_n - \frac{1}{2} \mu^{'n} V_{n \times n}^{-1} \mu^n
\end{eqnarray*}
\begin{eqnarray}
\mu^{'n} V_{n \times n}^{-1} \mu^n &=& [V_{n \times n}\{(I+K)_{n \times n}^{-1}y^n +\gamma (\delta^2I)^{-1}1^n\}]'V^{-1}_{n \times n}[V_{n \times n}\{(I+K)_{n \times n}^{-1}y^n +\gamma (\delta^2I)^{-1}1^n\}] \nonumber\\
&=& \bigg[y'^n(I+K)_{n \times n}^{-1} + 1'^n \frac{\gamma}{\delta^2}\bigg]V_{n \times n}'V_{n \times n}^{-1}V_{n \times n}\bigg[(I+K)_{n \times n}^{-1}y^n + \frac{\gamma}{\delta^2}1^n\bigg].
\end{eqnarray} 
Since $V$ is symmetric, i.e., $V' = V$, we have
\begin{eqnarray}
   \mu^{'n} V_{n \times n}^{-1} \mu^n &=& y'^n (I+K)_{n \times n}^{-1}V_{n \times n}(I+K)_{n \times n}^{-1}y^n + 2\frac{\gamma}{\delta^2}y'^n(I+K)_{n \times n}^{-1}V_{n \times n}1^n \nonumber\\
   && + \frac{\gamma^2}{\delta^4}1'^n V_{n \times n} 1'^n.
   \label{muvmu}
\end{eqnarray}
Using \eqref{muvmu} in \eqref{betastarn}, we get
\begin{eqnarray*}
    \beta^*_n &=& \beta + \frac{1}{2}y_n'(I+K)_{n \times n}^{-1}y + \frac{1}{2}\gamma^2 1_n[\delta^2I]^{-1}1_n \nonumber\\
   && - \frac{1}{2}\bigg[y'^n (I+K)_{n \times n}^{-1}V_{n \times n}(I+K)_{n \times n}^{-1}y^n + 2\frac{\gamma}{\delta^2}y'^n(I+K)_{n \times n}^{-1}V_{n \times n}1^n 
    + \frac{\gamma^2}{\delta^4}1'^n V_{n \times n} 1'^n\bigg] 
    \nonumber\\
    &=& \beta +\frac{1}{2} y'^n[(I+K)_{n \times n}^{-1} - (I+K)_{n \times n}^{-1}V_{n \times n}(I+K)_{n \times n}^{-1}]y^{n} - \frac{\gamma}{\delta^2}y'^n(I+K)_{n \times n}^{-1}V_{n \times n}1^n \nonumber\\
   && + \frac{n}{2}\frac{\gamma^2}{\delta^2} -\frac{1}{2}\frac{\gamma^2}{\delta^4}1'^n V_{n \times n} 1^n.
\end{eqnarray*}
So, \begin{eqnarray}
\label{betastartnplus1}
   \beta^*_{n+1} &=& \beta +\frac{1}{2} y'^{n+1}[(I+K)_{n+1 \times n+1}^{-1} - (I+K)_{n+1 \times n+1}^{-1}V_{n+1 \times n+1}(I+K)_{n+1 \times n+1}^{-1}]y^{n+1} \nonumber\\
   &&  - \frac{\gamma}{\delta^2}y'^{n+1}(I+K)_{n+1 \times n+1}^{-1}V_{n+1 \times n+1}1^{n+1} + \frac{n+1}{2}\frac{\gamma^2}{\delta^2} -\frac{1}{2}\frac{\gamma^2}{\delta^4}1'^{n+1} V_{n+1 \times n+1} 1^{n+1}.\nonumber\\  
\end{eqnarray}
Define
\begin{eqnarray}
\label{gamma1}
 \Gamma_{1,n+1 \times n+1} &=& (I+K)_{n+1 \times n+1}^{-1} - (I+K)_{n+1 \times n+1}^{-1}V_{n+1 \times n+1}(I+K)_{n+1 \times n+1}^{-1}  \\
  \label{gamma2}
  \Gamma_2^{n+1} &=& \frac{\gamma}{\delta^2}y'^{n+1}(I+K)_{n+1 \times n+1}^{-1}V_{n+1 \times n+1}1^{n+1} \\
  \label{capdel}
  \text{and }\Delta &=& \frac{n+1}{2}\frac{\gamma^2}{\delta^2} -\frac{1}{2}\frac{\gamma^2}{\delta^4}1'^{n+1}V_{n+1 \times n+1}1^n .
\end{eqnarray}
Using \eqref{gamma1}, \eqref{gamma2}, and \eqref{capdel} in \eqref{betastartnplus1}, we get
\begin{eqnarray}
\label{betanshrt}
    \beta^*_{n+1} &=& \beta + \frac{1}{2} y'^{n+1}\Gamma_{1, n+1 \times n+1} y^{n+1} -y'^{n+1}\Gamma_2^{n+1} + \Delta.
\end{eqnarray}
Now, we partition $y^{n+1}$, $\Gamma_{1, n+1 \times n+1}$, and $\Gamma_2^{n+1}$.  Write
\begin{eqnarray}
\label{ygamma1}
    y'^{n+1}\Gamma_{1, n+1 \times n+1} y^{n+1} &=& \begin{pmatrix} y^n & y_{n+1}\end{pmatrix} \left(\begin{array}{ccc}
\Gamma_{1,n \times n} & \vdots &g^n_1\\
\cdots & \vdots & \cdots\\
g'^n_1 & \vdots & \gamma_1 
\end{array}\right) \begin{pmatrix}y^n\\
y_{n+1}
\end{pmatrix} \nonumber\\
&=&  y'^{n}\Gamma_{1, n \times n} y^{n} +2y'^ng_1^n y_{n+1} +y^2_{n+1}\gamma_1
\end{eqnarray}
and 
\begin{eqnarray}
\label{ygamma2}
    y'^{n+1}\Gamma_2^{n+1} &=& \begin{pmatrix} y^n & y_{n+1}\end{pmatrix}\begin{pmatrix}\Gamma_2^n\\ 
\gamma_{2}
\end{pmatrix} = y'^n \Gamma_2^n + y_{n+1}\gamma_2.
\end{eqnarray}
Using \eqref{ygamma1} and \eqref{ygamma2}
 in \eqref{betanshrt}, we get
 \begin{eqnarray}
 \label{betastarrn1g1g2d}
     \beta^*_{n+1} &=& \beta+y'^{n}\Gamma_{1, n \times n} y^{n} +2y'^ng_1^n y_{n+1} +y^2_{n+1}\gamma_1 + y'^n \Gamma_2^n + y_{n+1}\gamma_2 + \Delta. \nonumber\\
     &=& \beta + \frac{1}{2}y'^n \Gamma_{1,n \times n}y^n - y'^n \Gamma_2^n + \Delta +\frac{1}{2}\gamma_1 y^2_{n+1} -y_{n+1}(\gamma_2 - y'^ng_1^n).
 \end{eqnarray}
We complete the square again. The terms in \eqref{betastarrn1g1g2d} containing $y_{n+1}$ become   
 \begin{eqnarray}
  &&   \frac{1}{2}\gamma_1 y^2_{n+1} -y_{n+1}(\gamma_2 - y'^ng_1^n) 
  \nonumber \\
  &&
\quad = \frac{1}{2}\gamma_1 \bigg[y^2_{n+1} -2y_{n+1} \frac{\gamma_2-y'^ng_1^n}{\gamma_1}+ \bigg(\frac{\gamma_2-y'^n g_1^n}{\gamma_1}\bigg)^2\bigg] -\frac{1}{2}\frac{(\gamma_2-y'^ng_1^n)^2}{\gamma_1} \nonumber\\
     && \quad = \frac{\gamma_1}{2} \bigg[y_{n+1}-\frac{\gamma_2-y'^ng_1^n}{\gamma_1}\bigg]^2 -\frac{1}{2\gamma_1}(\gamma_2-y'^ng_1^n)^2.
 \end{eqnarray}
 For brevity, let
 \begin{eqnarray}
 \label{A1A2}
 A_1 &=& \frac{\gamma_2-y'^ng_1^n}{\gamma_1} \nonumber\\
     A_2 &=& \frac{1}{2}y'^n \Gamma_{1,n \times n}y^n - y'^n \Gamma_2^n + \Delta  -\frac{1}{2\gamma_1}(\gamma_2-y'^ng_1^n)^2.
 \end{eqnarray}
 Using \eqref{A1A2} in \eqref{betastarrn1g1g2d}, we have
 \begin{eqnarray*}
     \beta^*_{n+1} &=& \beta + \frac{\gamma_1}{2}(y_{n+1}-A_1)^2 +A_2.
 \end{eqnarray*}
 Now, since $m(y^n)$ is the marginal density of $y^n$ and, $m(y^{n+1})$ is the marginal density of $y^{n+1}$,
 \begin{eqnarray}
     \int_{\mathbb{R}}^{} \frac{m(y^{n+1})}{m(y^n)} \rm{d}y_{n+1} = 1.
 \end{eqnarray}
 From \eqref{cbetastar} we have that
 \begin{eqnarray}
 \label{csoln}
     \int_{\mathbb{R}}^{} c \times \frac{{\beta^*_{n+1}}^{-\big(\alpha+\frac{n+1}{2}\big)}}{{\beta^*_{n}}^{-\big(\alpha+\frac{n}{2}\big)}} \rm{d}y_{n+1} &=& 1 .\end{eqnarray}
     So solving for $c$ gives
     \begin{eqnarray*}
   c &=& \frac{{\beta^*_{n}}^{-\big(\alpha+\frac{n}{2}\big)}}{\int_{\mathbb{R}}^{}{\beta^*_{n+1}}^{-\big(\alpha+\frac{n+1}{2}\big)}\rm{d}y_{n+1}}.
 \end{eqnarray*}
 Using \eqref{csoln} in \eqref{cbetastar}, we have
 \begin{eqnarray}
 \label{marginalratios}
     \frac{m(y^{n+1})}{m(y^n)} &=& \frac{{\beta^*_{n}}^{-\big(\alpha+\frac{n}{2}\big)}}{\int_{\mathbb{R}}^{}{\beta^*_{n+1}}^{-\big(\alpha+\frac{n+1}{2}\big)}\rm{d}y_{n+1}} \times \frac{(\beta_{n+1}^*)^{-(\alpha+\frac{n+1}{2})}}{(\beta_{n}^*)^{-(\alpha+\frac{n}{2})}}\nonumber\\
     &=& \frac{{\beta^*_{n+1}}^{-\big(\alpha+\frac{n+1}{2}\big)}}{\int_{\mathbb{R}}^{}{\beta^*_{n+1}}^{-\big(\alpha+\frac{n+1}{2}\big)}\rm{d}y_{n+1}}.
 \end{eqnarray}
 Now, 
 \begin{eqnarray}
  \label{betastarn1power}
      {\beta^*_{n+1}}^{-(\alpha+\frac{n+1}{2})} &=& \bigg[\beta + \frac{\gamma_1}{2}(y_{n+1}-A_1)^2 +A_2\bigg]^{-(\alpha+\frac{n+1}{2})}.
 \end{eqnarray}
 Rename, $\beta^{**} = \beta + A_2 $. Then, \eqref{betastarn1power} becomes
 \begin{eqnarray}
 \label{betasnarn1power2}
      {\beta^*_{n+1}}^{-(\alpha+\frac{n+1}{2})} &=& \bigg[\beta^{**} + \frac{\gamma_1}{2}(y_{n+1}-A_1)^2\bigg]^{-(\alpha+\frac{n+1}{2})} \nonumber\\
      &=& {\beta^{**}}^{-\big(\alpha+\frac{n}{2}\big)} {\beta^{**}}^{-\frac{1}{2}} \bigg[1+\frac{\gamma_1}{2\beta^{**}}(y_{n+1}-A_1)^2\bigg]^{-\big(\alpha + \frac{n+1}{2}\big)}.
 \end{eqnarray}
 By definition, the t-density is given by
 \begin{eqnarray}
 \label{tpdf}
     St_v(\tau, \Sigma) (g) &=& \frac{\Gamma(\frac{v+d}{2})}{\Gamma(\frac{v}{2})\pi^{\frac{d}{2}}|v\Sigma|^\frac{1}{2}}\Bigl(1+\frac{(g-\tau)'\Sigma^{-1}(g-\tau)}{v}\Bigl)^{-\frac{v+d}{2}}.
 \end{eqnarray}
 So if we let
 \begin{equation}
 \label{studentstterms}
      v = 2\alpha, d = 1, \Sigma = \frac{\beta^{**}}{\frac{2\alpha+n}{2}}\frac{1}{\gamma_1}, g = y_{n+1}, \tau = A_1.
 \end{equation}
and use \eqref{studentstterms} in \eqref{tpdf}, we get
\begin{eqnarray*}
    St_{2\alpha+n}\bigg(A_1, \frac{\beta^{**}}{\frac{2\alpha+n}{2}}\bigg)(y_{n+1})
    &=& \frac{\Gamma(\frac{2\alpha+n+1}{2})}{\Gamma(\frac{2\alpha+n}{2})\pi^{\frac{1}{2}}\bigg|(2\alpha+n)\frac{\beta^{**}}{\frac{2\alpha+n}{2}}\frac{1}{\gamma_1}\bigg|^\frac{1}{2}} \nonumber\\
    &&\times \Bigl(1+\frac{(y_{n+1}-A_1)'\bigg(\frac{\beta^{**}}{\frac{2\alpha+n}{2}}\frac{1}{\gamma_1}\bigg)^{-1}(y_{n+1}-A_1)}{2\alpha+n}\Bigl)^{-\frac{2\alpha+n+1}{2}} \nonumber\\
    &=& \frac{\Gamma(\frac{2\alpha+n+1}{2})}{\Gamma(\frac{2\alpha+n}{2})} \gamma_1^{\frac{1}{2}} \frac{1}{(2\pi)^{\frac{1}{2}}} {\beta^{**}}^{-\frac{1}{2}}\bigg[1+\frac{\gamma_1}{2\beta^{**}}(y_{n+1}-A_1)^2\bigg]^{-\frac{2\alpha+n+1}{2}}.
\end{eqnarray*}
Hence, \begin{eqnarray}
\label{fntdist}
 &&   {\beta^{**}}^{-\frac{1}{2}}\bigg[1+\frac{\gamma_1}{2\beta^{**}}(y_{n+1}-A_1)^2\bigg]^{-\big(\alpha+\frac{n+1}{2}\big)} 
    \nonumber \\
    && \quad = \frac{\Gamma(\frac{2\alpha+n+1}{2})}{\Gamma(\frac{2\alpha+n}{2})}\frac{(2\pi)^{\frac{1}{2}}}{\gamma_1^{\frac{1}{2}}} \times St_{2\alpha+n}\bigg(A_1, \frac{\beta^{**}}{\frac{2\alpha+n}{2}}\bigg)(y_{n+1}).\nonumber\\
\end{eqnarray}
Using \eqref{fntdist} in \eqref{betasnarn1power2}, and \eqref{betasnarn1power2} in \eqref{marginalratios}, we have
\begin{eqnarray}
\label{marginalratios2} 
    \frac{m(y^{n+1})}{m(y^n)} &=& \frac{{\beta^{**}}^{-(\alpha+\frac{n}{2})}}{\int_{\mathbb{R}}^{}{\beta^*_{n+1}}^{-(\alpha+\frac{n+1}{2})}\rm{d}y_{n+1}} \frac{\Gamma(\frac{2\alpha+n}{2})}{\Gamma(\frac{2\alpha+n+1}{2})} \frac{(2\pi)^{\frac{1}{2}}}{(\gamma_1)^{\frac{1}{2}}} \nonumber\\
    && \times St_{2\alpha+n}\bigg(A_1, \frac{\beta^{**}}{\frac{2\alpha+n}{2}}\bigg)(y_{n+1}).
\end{eqnarray}
Since $\frac{m(y^{n+1})}{m(y^n)} = m(y_{n+1}|y^n)$ is a density,  $\int_{\mathbb{R}}^{}\frac{m(y^{n+1})}{m(y^n)}\rm{d}y_{n+1} = 1.$
Integrating the right hand side of\eqref{marginalratios2} w.r.t $y_{n+1}$ gives
that
\begin{eqnarray}
\label{intwrtyn1}
    \frac{{\beta^{**}}^{-(\alpha+\frac{n}{2})}}{\int_{\mathbb{R}}^{}{\beta^*_{n+1}}^{-(\alpha+\frac{n+1}{2})}\rm{d}y_{n+1}} \frac{\Gamma(\frac{2\alpha+n}{2})}{\Gamma(\frac{2\alpha+n+1}{2})} \frac{(2\pi)^{\frac{1}{2}}}{(\gamma_1)^{\frac{1}{2}}} 
     \int_{\mathbb{R}}^{} St_{2\alpha+n}\bigg(A_1, \frac{\beta^{**}}{\frac{2\alpha+n}{2}}\bigg)(y_{n+1})\rm{d}y_{n+1}
\end{eqnarray} 
equals 1, since $y_{n+1}$ is only in the argument of the $t$ distribution. The integral of the $t$ distribution being one means \eqref{intwrtyn1} gives
\begin{eqnarray}
\label{answrtyn1}
    \frac{{\beta^{**}}^{-(\alpha+\frac{n}{2})}}{\int_{\mathbb{R}}^{}{\beta^*_{n+1}}^{-(\alpha+\frac{n+1}{2})}\rm{d}y_{n+1}} \frac{\Gamma(\frac{2\alpha+n}{2})}{\Gamma(\frac{2\alpha+n+1}{2})} \frac{(2\pi)^{\frac{1}{2}}}{(\gamma_1)^{\frac{1}{2}}} = 1.
\end{eqnarray}
Finally using \eqref{answrtyn1} in \eqref{marginalratios2}, we get
\begin{eqnarray*}
   m(y_{n+1}|y^n) &=& St_{2\alpha+n}\bigg(A_1, \frac{\beta^{**}}{\frac{2\alpha+n}{2}}\bigg)(y_{n+1}) . \quad \square
\end{eqnarray*}
\end{proof}

\subsection{Proof of Theorem \ref{gammadelta}}

\begin{proof}
          The joint likelihood of $y^n$ and $a^n$ given $\sigma^2$, $\gamma$ and $\delta^2$ takes the form
\begin{eqnarray*}
  \mathcal{L}_2(y^n,a^n|\sigma^2,\gamma,\delta^2) &=& \frac{1}{(2\pi)^{\frac{n}{2}}(\sigma^2)^{\frac{n}{2}}|I_{n \times n}+K_{n \times n}|^\frac{1}{2}}e^{-\frac{1}{2\sigma^2}(y^n-a^n)^{'}(I_{n \times n}+K_{n \times n})^{-1}(y^n-a^n)}   
  \nonumber\\
  && \times\frac{1}{(2\pi)^{\frac{n}{2}}(\sigma^2\delta^2)^\frac{n}{2}} e^{-\frac{1}{2\sigma^2}(a^n-\gamma1^n)^{'}(\delta^2I_{n \times n})^{-1}(a^n-\gamma1^n)} .
  \nonumber\\
\end{eqnarray*}
Integrating out $a^n$ using the density of $(a^n|\sigma^2, \gamma, \delta^2)$ gives
\begin{eqnarray}
&& \mathcal{L}_3(y^n|\gamma, \delta^2, \sigma^2) 
\nonumber \\
&=& \quad \frac{1}{(2\pi)^{\frac{n}{2}}(\sigma^2)^{\frac{n}{2}}|I_{n \times n}+K_{n \times n}|^\frac{1}{2}} \frac{1}{(2\pi)^{\frac{n}{2}}(\sigma^2\delta^2)^\frac{n}{2}}
\nonumber \\
&& \times \int_{\mathbf{R}}^{} e^{-\frac{1}{2\sigma^2}(y^n-a^n)^{'}(I_{n \times n}+K_{n \times n})^{-1}(y^n-a^n)} e^{-\frac{1}{2\sigma^2}(a^n-\gamma1^n)^{'}(\delta^2I_{n \times n})^{-1}(a^n-\gamma1^n)} \rm{d}a^n.
\nonumber\\
&=& \frac{1}{(2\pi)^n(\sigma^2)^n|I_{n \times n}+K_{n \times n}|^\frac{1}{2}(\delta^2)^\frac{n}{2}} 
\nonumber\\
&&\times\int_{\mathbf{R}}^{} \Bigl\{e^{-\frac{1}{2\sigma^2}[y^{'n}(I_{n \times n}+K_{n \times n})^{-1}y-a^{'n}(I_{n \times n}+K_{n \times n})^{-1}y^n - y^{'n}(I_{n \times n}+K_{n \times n})^{-1}a^n]} 
\nonumber\\
 && \times e^{-\frac{1}{2\sigma^2}[a^{'n}(I_{n \times n}+K_{n \times n})^{-1}a^n+a^{'n}(\delta^2I_{n \times n})^{-1}a^n-\gamma1^{'}(\delta^2I_{n \times n})^{-1}1^n -\gamma a^{'n}(\delta^2I_{n \times n})^{-1}1^n +\gamma^21{'n}(\delta^2I_{n \times n})^{-1}1^n]}\Bigr\} \rm{d}a^n
\nonumber\\
&=& \frac{e^{-\frac{1}{2\sigma^2}[y^{'n}(I_{n \times n}+K_{n \times n})^{-1}y^n+\gamma^21^{'n}(\delta^2I_{n \times n})^{-1}1^n]}}{(2\pi)^n(\sigma^2)^n|I_{n \times n}+K_{n \times n}|^\frac{1}{2}(\delta^2)^{\frac{n}{2}}} 
\nonumber \\
&& \times \int_{\mathbf{R}}^{} \Biggl\{e^{-\frac{1}{2\sigma^2}a^{'n}\Bigl\{(I_{n \times n}+K_{n \times n})^{-1}+(\delta^2I_{n \times n})^{-1}\Bigr\}a^n} 
\nonumber\\ 
&&  \times e^{-\frac{1}{2\sigma^2}\Bigl[-2a^{'n}\Bigl\{(I_{n \times n}+K_{n \times n})^{-1}+(\delta^2I_{n \times n})^{-1}\Bigr\}\Bigl\{(I_{n \times n}+K_{n \times n})^{-1}+(\delta^2I_{n \times n})^{-1}\Bigr\}^{-1}\Bigr]} 
\nonumber\\
&& \times e^{-\frac{1}{2\sigma^2}\Bigl[-2a^{'n}{\Bigl\{(I_{n \times n}+K_{n \times n})^{-1}y^n+(\delta^2I_{n \times n})^{-1}1^n\Bigr\}}\Bigr]}\Biggr\} \rm{d}a^n .
\nonumber\\
\end{eqnarray}

To prove both Clause \#1 and Clause \#2 we derive a simplified form of $\mathcal{L}_3(y^n|\sigma^2, \gamma, \delta^2)$.
We start by rewriting $L_3(y^n|\sigma^2, \gamma, \delta^2)$ by substituting from the expressions for $V_{n \times n}$ and $\mu^n$ from \eqref{V} and \eqref{mu}.
That is, from \eqref{V} and \eqref{mu} we get
\begin{eqnarray}
\label{Vgd}
     V_{n \times n} &=& [(I_{n \times n}+K_{n \times n})^{-1}+(\delta^2I_{n \times n})^{-1}]^{-1}\\
    \mu^n &=& V_{n \times n}[(I_{n \times n}+K_{n \times n})^{-1}y^n +\gamma (\delta^2I_{n \times n})^{-1}1^n],
    \label{mugd}
    \nonumber\\
\end{eqnarray}
and hence
\begin{eqnarray}
\label{incompsq}
\mathcal{L}_3(y^n|\sigma^2, \gamma, \delta^2) &=& \frac{e^{-\frac{1}{2\sigma^2}[y^{'n}(I_{n \times n}+K_{n \times n})^{-1}y^n+\gamma^21^{'n}(\delta^2I_{n \times n})^{-1}1^n]}}{(2\pi)^n(\sigma^2)^n|I_{n \times n}+K_{n \times n}|^\frac{1}{2}(\delta^2)^{\frac{n}{2}}} 
\nonumber \\
&& \quad \times \int_{\mathbf{R}}^{} e^{-\frac{1}{2\sigma^2}\Bigl[a^{'n}V_{n \times n}^{-1}a^n - 2a^{'n}V_{n \times n}^{-1}\undertilde{\mu^n}\Bigr]} \rm{d}a^n .
\nonumber\\
\end{eqnarray}
We complete the square in the exponent in the integral by multiplying and dividing \eqref{incompsq} by  $e^{-\frac{1}{2\sigma^2}}\mu^{'n}V_{n \times n}^{-1}\mu^n$. This gives
\begin{eqnarray}
\mathcal{L}_3(y^n|\gamma, \delta^2, \sigma^2) &=& \frac{e^{-\frac{1}{2\sigma^2}[y^{'n}(I_{n \times n}+K_{n \times n})^{-1}y^n+\gamma^21^{'n}(\delta^2I_{n \times n})^{-1}1^n]}}{(2\pi)^n(\sigma^2)^n|I_{n \times n}+K_{n \times n}|^\frac{1}{2}(\delta^2)^{\frac{n}{2}}} e^{\frac{1}{2\sigma^2}\mu^{'n}V_{n \times n}^{-1}\undertilde{\mu^n}}
\nonumber \\
&& \quad \times \int_{\mathbf{R}}^{} e^{-\frac{1}{2\sigma^2}(a^{'n}V_{n \times n}^{-1}a^n - 2a^{'n}V_{n \times n}^{-1}\mu+\mu^{'}V_{n \times n}^{-1}\mu)} \rm{d}a^n.
\nonumber\\
&=& \frac{e^{-\frac{1}{2\sigma^2}[y'(I_{n \times n}+K_{n \times n})^{-1}y+\gamma^21'(\delta^2I_{n \times n})^{-1}1]}}{(2\pi)^n(\sigma^2)^n|I_{n \times n}+K_{n \times n}|^\frac{1}{2}(\delta^2)^{\frac{n}{2}}} e^{\frac{1}{2\sigma^2}^{'}V_{n \times n}^{-1}\mu^n} 
\nonumber \\
&& \quad \times \int_{\mathbf{R}}^{} e^{-\frac{1}{2\sigma^2}(a^n-\mu^n)^{'}V_{n \times n}^{-1}(a^n-\mu^n)} \rm{d}a^n.
\nonumber\\
\label{compsq}
\end{eqnarray}
The integral in \eqref{compsq} becomes $1$ if we divide and multiply by $(2\pi)^{\frac{n}{2}}(\sigma^2)^\frac{n}{2}|V_{n \times n}|^\frac{1}{2}$, i.e., we get
\begin{eqnarray}
\mathcal{L}_3(y^n|\gamma, \delta^2, \sigma^2)&=& \frac{e^{-\frac{1}{2\sigma^2}[y^{'n}(I+K)_{n \times n}^{-1}y^n+\gamma^21^{'n}(\delta^2I_{n \times n})^{-1}1^n]}}{(2\pi)^n(\sigma^2)^n|(I+K)_{n \times n}|^\frac{1}{2}(\delta^2)^{\frac{n}{2}}}  e^{\frac{1}{2\sigma^2}\mu^{'n}V_{n \times n}^{-1}\mu^n} (2\pi)^{\frac{n}{2}}(\sigma^2)^\frac{n}{2}|V_{n \times n}|^\frac{1}{2}
\nonumber\\
&& \times \Biggl[\frac{1}{(2\pi)^{\frac{n}{2}}(\sigma^2)^\frac{n}{2}|V_{n \times n}|^\frac{1}{2}}\int_{\mathbf{R}}^{} e^{-\frac{1}{2\sigma^2}(a^n-\mu^n)^{'}V_{n \times n}^{-1}(a^n-\mu^n)} \rm{d}a^n.\Biggr]
\nonumber\\
&=&\frac{(2\pi)^{\frac{n}{2}}(\sigma^2)^\frac{n}{2}|V_{n \times n}|^\frac{1}{2}}{(2\pi)^n (\sigma^2)^n |(I+K)_{n \times n}|^{\frac{1}{2}}(\delta^2)^{\frac{n}{2}}} 
\nonumber \\
&& \quad \times e^{\frac{1}{2\sigma^2}\mu^{'n}V_{n \times n}^{-1}\mu^n} e^{-\frac{1}{2\sigma^2}[y^{'n}(I+K)_{n \times n}^{-1}y^n+\gamma^21^{'n}(\delta^2I_{n \times n})^{-1}1^n]}.
\label{L3}
\end{eqnarray}

To prove Clause \#1, we collect the terms that depend on $\gamma$ from \eqref{L3}. Note that, the factor 
\begin{eqnarray}
  \frac{(2\pi)^{\frac{n}{2}}(\sigma^2)^\frac{n}{2}|V_{n \times n}|^\frac{1}{2}}{(2\pi)^n (\sigma^2)^n |(I+K)_{n \times n}|^{\frac{1}{2}}(\delta^2)^{\frac{n}{2}}}  e^{-\frac{1}{2\sigma^2}[y^{'n}(I+K)^{-1}y^n]}
  \label{nogamma1}
\end{eqnarray} from \eqref{L3} does not depend on $\gamma$. From Eqn \eqref{mugd}, we know that $\mu$ contains $\gamma$, so the part of \eqref{L3} that contains $\gamma$ is:
\begin{eqnarray}
&& e^{-\frac{1}{2\sigma^2}[-\mu^{'n}V_{n \times n}^{-1}\mu^n + \gamma^21{'n}(\delta^2I)^{-1}1^n]} 
\nonumber \\
&& =
 e^{-\frac{1}{2\sigma^2}\Bigl[-\Bigl\{(I+K)_{n \times n}^{-1}y^n+\gamma(\delta^2I_{n \times n})^{-1}1^n\Bigr\}^{'} V_{n \times n}^{'} V_{n \times n}^{-1} V_{n \times n} \Bigl\{(I+K)_{n \times n}^{-1}y^n+\gamma(\delta^2I_{n \times n})^{-1}1^n\Bigr\}\Bigr]} \nonumber\\
&& \quad \times  e^{-\frac{1}{2\sigma^2}\gamma^21^{'n}(\delta^2I_{n \times n})^{-1}1^n}.
\end{eqnarray}
Here $V_{n \times n}^{'} = V_{n \times n}.$ So,
\begin{eqnarray}
&& e^{-\frac{1}{2\sigma^2}[-\mu^{'n}V_{n \times n}^{-1}\mu^n + \gamma^21^{'n}(\delta^2I_{n \times n})^{-1}1^n]} 
\nonumber \\
&& \quad = e^{-\frac{1}{2\sigma^2}\Bigl[-\Bigl\{(I+K)^{-1}y+\gamma(\delta^2I_{n \times n})^{-1}1^n\Bigr\}^{'}  V_{n \times n} \Bigl\{(I+K)^{-1}y+\gamma(\delta^2I)^{-1}1^n\Bigr\} +\gamma^21{'}(\delta^2I_{n \times n})^{-1}1^n\Bigr]}.
\nonumber\\
&& \quad = e^{\frac{1}{2\sigma^2}\Bigl[y^{'n}(I+K)_{n \times n}^{-1}V_{n \times n}(I+K)_{n \times n}^{-1}y^n + \gamma1^{'n}(\delta^2I_{n \times n})^{-1}V_{n \times n}(I+K)_{n \times n}^{-1}y^n\Bigr]}  \nonumber\\
&& \quad \times e^{\frac{1}{2\sigma^2}\Bigl[\gamma y^{'n}(I+K)_{n \times n}^{-1}V_{n \times n}(\delta^2I_{n \times n})^{-1}1+\gamma^21^{'n}(\delta^2I_{n \times n})^{-2}V_{n \times n}1^n\Bigr]} e^{-\frac{1}{2\sigma^2}\Bigl[\gamma^21^{'n}(\delta^2I_{n \times n})^{-1}1^n\Bigr]}.
\label{hgamma}
\nonumber\\
\end{eqnarray}
We  name the terms containing $\gamma$ from the expression in right hand side of \eqref{hgamma} as $h(\gamma)$ .
\begin{eqnarray}
h(\gamma) = e^{-\frac{1}{2\sigma^2}\Bigl[-2\gamma y^{'n}\frac{(I+K)_{n \times n}^{-1}V^n}{\delta^2}1^n + \gamma^21^{'n}\Bigl(\frac{I_{n \times n}}{\delta^2}-\frac{V_{n \times n}}{\delta^4}\Bigr)1^n\Bigr]}
\nonumber\\
\label{fngamma}
\end{eqnarray}
The term that does not contain $\gamma$ in Eqn \eqref{hgamma} is 
\begin{eqnarray}
  e^{\frac{1}{2\sigma^2}\Bigl[y^{'n}(I+K)_{n \times n}^{-1}V_{n \times n}(I+K)_{n \times n}^{-1}y^n \Bigr]}
\nonumber\\
\label{nogamma2}
\end{eqnarray}
The product of \eqref{fngamma} that contains $\gamma$ with \eqref{nogamma1} and \eqref{nogamma2} that do not contain $\gamma$ is \eqref{statehgamma}.

To prove Clause \#2, we separate the factors with $\delta^2$ and without $\delta^2$ in \eqref{L3}. From \eqref{Vgd}, we know that $V_{n \times n}$ contains $\delta^2$. The part of \eqref{L3} that depends on $\delta^2$ is $\frac{|V_{n \times n}|}{\delta^{2^{\frac{n}{2}}}} \times$ the factor in \eqref{hgamma}. Substituting for $V_{n \times n}$ from \eqref{Vgd}, we have that $g(\delta^2)$
equals
\begin{eqnarray}
  && \frac{\Bigl|\Bigl\{(I+K)_{n \times n}^{-1}+(\delta^2I_{n \times n})^{-1}\Bigr\}\Bigr|^{\frac{1}{2}}}{(\delta^2)^{\frac{n}{2}}} \nonumber\\
  && e^{\frac{1}{2\sigma^2}\Bigl[y^{'n}(I+K)_{n \times n}^{-1}\bigl\{(I+K)_{n \times n}^{-1}+(\delta^2I_{n \times n})^{-1}\bigr\}^{-1}(I+K)_{n \times n}^{-1}y^n + \frac{2\gamma}{\delta^2} y^{'n}(I+K)_{n \times n}^{-1}\bigl\{(I+K)_{n \times n}^{-1}+(\delta^2I_{n \times n})^{-1}\bigr\}^{-1}1^n\Bigr]} 
  \nonumber\\
  && e^{-\frac{1}{2\sigma^2}\Bigl[\frac{\gamma^2}{\delta^4}1^{'n}\bigl\{(I+K)_{n \times n}^{-1}+(\delta^2I_{n \times n})^{-1}\bigr\}^{-1}1^n-\frac{\gamma^2}{\delta^2}1^{'n}1^n\Bigr]}.
  \nonumber\\
  \label{delta}
\end{eqnarray}The remaining factors in \eqref{L3} are: 
\begin{eqnarray}
  \frac{(2\pi)^{\frac{n}{2}}(\sigma^2)^\frac{n}{2}}{(2\pi)^n (\sigma^2)^n |(I+K)_{n \times n}|^{\frac{1}{2}}} e^{-\frac{1}{2\sigma^2}[y^{'n}(I+K)_{n \times n}^{-1}y^n]}
  \label{nodelta}
\end{eqnarray}. 
The product of \eqref{delta} and \eqref{nodelta} gives \eqref{stategdelta}.
      \end{proof}

\end{document}